\documentclass{article}

\usepackage[hidelinks]{hyperref}       
\usepackage{url}            
\usepackage{booktabs}       
\usepackage{amsfonts}       
\usepackage{nicefrac}       
\usepackage{microtype}      
\usepackage{algpseudocode}
\usepackage{comment}
\usepackage{enumitem}

\usepackage{wrapfig,tikz}
\usepackage{amsmath,longtable,fancyhdr,booktabs,multirow}
\usepackage{graphicx,float}
\usepackage{amssymb,xcolor,amsthm}
\usepackage{subfigure}
\usepackage{color}
\usepackage{colortbl}
\usepackage{theoremref}
\usepackage{amsmath}

\usepackage[accepted]{icml2018}

\newcommand{\mbf}[1]{\mathbf{#1}}

\newcommand{\mbb}[1]{\mathbb{#1}}

\newcommand{\ud}{\mathrm{d}}

\newcommand{\mcal}{\mathcal}
\newcommand{\norm}[1]{\left\lVert#1\right\rVert}

\newtheorem{theorem}{Theorem}

\newtheorem{proposition}{Proposition}

\newtheorem{innercustomthm}{Theorem}

\newenvironment{customthm}[1]
{\renewcommand\theinnercustomthm{#1}\innercustomthm}
{\endinnercustomthm}

\newtheorem{innercustomprop}{Proposition}
\newenvironment{customprop}[1]
{\renewcommand\theinnercustomprop{#1}\innercustomprop}
{\endinnercustomprop}

\newcommand{\be}{\begin{equation}}
\newcommand{\ee}{\end{equation}}
\definecolor{Gray}{gray}{0.85}
\definecolor{LightCyan}{rgb}{0.88,1,1}

\makeatletter
\usepackage{xspace} 
\def\@onedot{\ifx\@let@token.\else.\null\fi\xspace}
\DeclareRobustCommand\onedot{\futurelet\@let@token\@onedot}

\newcommand{\figref}[1]{Figure~\ref{#1}}

\newcommand{\secref}[1]{Section~\ref{#1}}

\newcommand{\thmref}[1]{Theorem~\ref{#1}}

\newcommand{\propref}[1]{Proposition~\ref{#1}}

\def\eg{\emph{e.g}\onedot}

\def\ie{\emph{i.e}\onedot}

\def\etc{\emph{etc}\onedot}

\def\wrt{w.r.t\onedot}

\usepackage{natbib}
\icmltitlerunning{Accelerating Natural Gradient with Higher-Order Invariance}
\begin{document}
\twocolumn[
\icmltitle{Accelerating Natural Gradient with \\ Higher-Order Invariance}
\begin{icmlauthorlist}
	\icmlauthor{Yang Song}{s}
	\icmlauthor{Jiaming Song}{s}
	\icmlauthor{Stefano Ermon}{s}
\end{icmlauthorlist}

\icmlaffiliation{s}{Computer Science Department, Stanford University}

\icmlcorrespondingauthor{Yang Song}{yangsong@cs.stanford.edu}
\icmlcorrespondingauthor{Jiaming Song}{tsong@cs.stanford.edu}
\icmlcorrespondingauthor{Stefano Ermon}{ermon@cs.stanford.edu}

\icmlkeywords{Natural Gradient, Invariance, Riemannian Geometry, Numerical Integrators}

\vskip 0.3in
]



\printAffiliationsAndNotice{}  

\begin{abstract}

An appealing property of the natural gradient is that it is invariant to arbitrary differentiable reparameterizations of the model. However, this invariance property requires infinitesimal steps and is lost in practical implementations with small but finite step sizes. In this paper, we study invariance properties from a combined perspective of Riemannian geometry and numerical differential equation solving. We define the order of invariance of a numerical method to be its convergence order to an invariant solution. We propose to use higher-order integrators and geodesic corrections to obtain more invariant optimization trajectories. We prove the numerical convergence properties of geodesic corrected updates and show that they can be as computational efficient as plain natural gradient. Experimentally, we demonstrate that invariance leads to faster optimization and our techniques improve on traditional natural gradient in deep neural network training and natural policy gradient for reinforcement learning.
\end{abstract}
\section{Introduction}

Non-convex optimization is a key component of the success of deep learning. Current state-of-the-art training methods are usually variants of stochastic gradient descent (SGD), such as AdaGrad~\citep{duchi2011adaptive}, RMSProp~\citep{hintonlecture} and Adam~\citep{kingma2014adam}.
While generally effective, performance of those first-order optimizers is highly dependent on the curvature of the optimization objective. When the Hessian matrix of the objective at the optimum has a large condition number, the problem is said to have pathological curvature~\citep{martens2010deep,sutskever2013importance}, and first-order methods will have trouble in making progress. The curvature, however, depends on how the model is parameterized. There may be some equivalent way of parameterizing the same model which has better-behaved curvature and is thus easier to optimize with first-order methods. Model reparameterizations, such as good network architectures~\citep{simonyan2014very, he2016deep} and normalization techniques \citep{lecun2012efficient,ioffe2015batch,salimans2016weight} are often critical for the success of first-order methods.

The natural gradient~\citep{amari1998natural} method
takes a different perspective to the same problem. Rather than devising a different parameterization for first-order optimizers, it tries to make the optimizer itself invariant to reparameterizations by directly operating on the manifold of probabilistic models. 
This invariance, however, only holds in the idealized case of infinitesimal steps, \ie, for continuous-time natural gradient descent trajectories on the manifold~\citep{ollivier2013riemannian,ollivier2015riemannian}. Practical implementations with small but finite step size (learning rate) are only approximately invariant. 
Inspired by Newton-Raphson method, the learning rate of natural gradient method is usually set to values near 1 in real applications~\citep{martens2010deep,martens2014new}, leading to potential loss of invariance. 

In this paper, we investigate invariance properties within the framework of Riemannian geometry and numerical differential equation solving. We observe that both the exact solution of the natural gradient dynamics and its approximation obtained with Riemannian Euler method~\citep{Bielecki2002-ix} are invariant. We propose to measure the invariance of a numerical scheme by \textit{studying its rate of convergence to those idealized truly invariant solutions}. It can be shown that the traditional natural gradient update (based on the forward Euler method) converges in first order. For improvement, we first propose to use a second-order Runge-Kutta integrator. Additionally, we introduce corrections based on the geodesic equation. We argue that the Runge-Kutta integrator converges to the exact solution in second order, and the method with geodesic corrections converges to the Riemannian Euler method in second order. Therefore, all the new methods have higher order of invariance, and experiments verify their faster convergence in deep neural network training and policy optimization for deep reinforcement learning. Moreover, the geodesic correction update has a faster variant which keeps the second-order invariance while being roughly as time efficient as the original natural gradient update. Our new methods can be used as drop-in replacements in any situation where natural gradient may be used.

\section{Preliminaries}
\subsection{Riemannian Geometry and Invariance}\label{sec:pre}

We use \emph{Einstein's summation convention} throughout this paper to simplify formulas. The convention states that when any index variable appears twice in a term, once as a superscript and once as a subscript, it indicates summation of the term over all possible values of the index variable. For example, $a^\mu b_\mu \triangleq \sum_{\mu=1}^n a^\mu b_\mu$ when index variable $\mu \in [n]$. 

Riemannian geometry is used to study intrinsic properties of differentiable manifolds equipped with metrics. The goal of this necessarily brief section is to introduce some key concepts related to the understanding of invariance. For more details, please refer to~\cite{petersen2006riemannian} and~\cite{amari1987differential}.

\begin{figure}
	\centering
		\includegraphics[width=0.95\columnwidth]{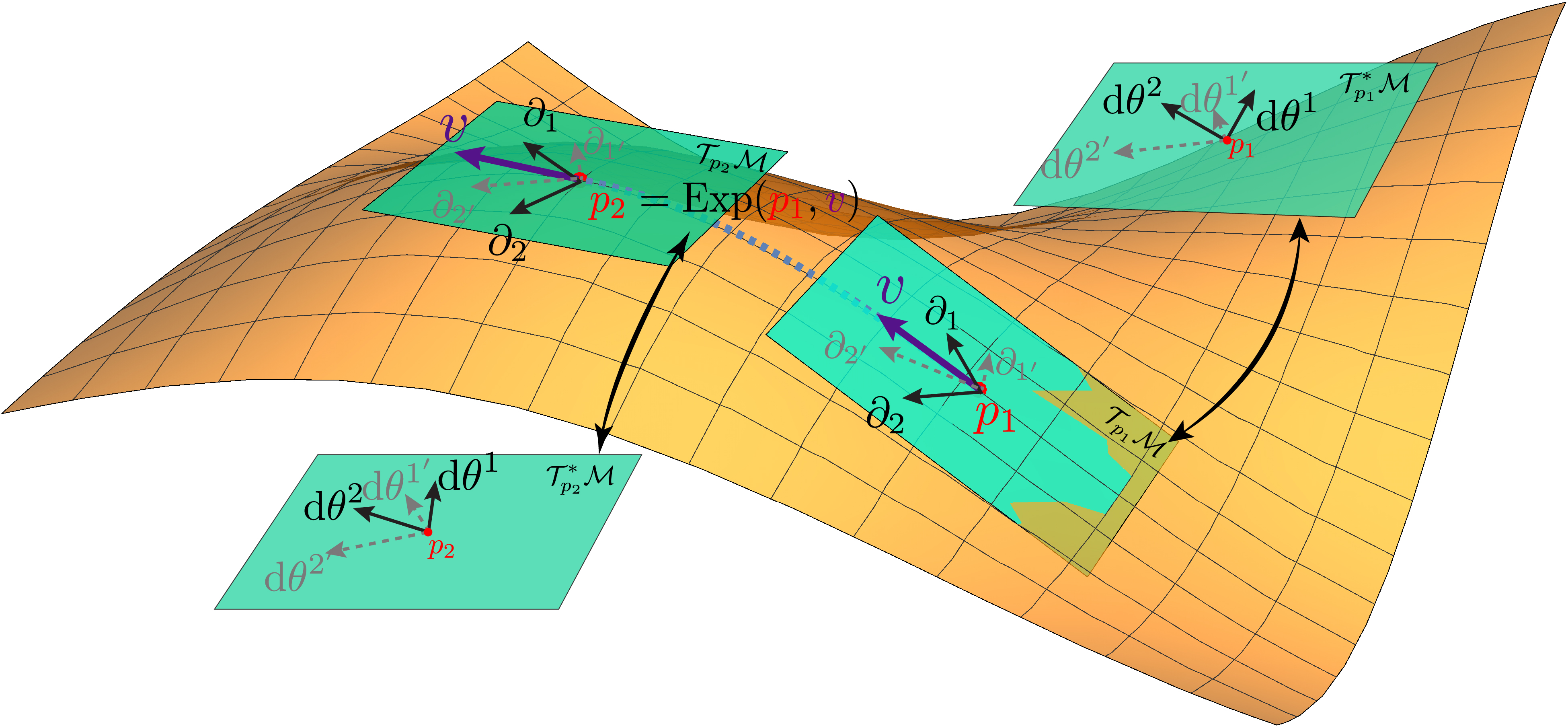}
	\caption{An illustration of Riemannian geometry concepts: tangent spaces, cotangent spaces, coordinate basis, dual coordinate basis, geodesics and the exponential map.}
	\label{fig:geometry}
\end{figure}

In this paper, we describe a family of probabilistic models as a manifold. Roughly speaking, a \emph{manifold} $\mcal{M}$ of dimension $n$ is a smooth space whose local regions resemble $\mbb{R}^n$~\citep{carroll2004spacetime}. Assume there exists a smooth mapping $\phi:\mcal{M}\rightarrow \mbb{R}^n$ in some neighborhood of $p$ and for any $p\in\mcal{M}$, $\phi(p)$ is the coordinate of $p$. As an example, if $p$ is a parameterized distribution, $\phi(p)$ will refer to its parameters. There is a linear space associated with each $p \in \mcal{M}$ called the \emph{tangent space} $\mcal{T}_p\cal{M}$. Each element $v \in \mcal{T}_p\cal{M}$ is called a \emph{vector}. For any tangent space $T_p \mcal{M}$, there exists a dual space $T_p^* \mcal{M}$ called the \emph{cotagent space}, which consists of all linear real-valued functions on the tangent space. Each element $v^*$ in the dual space $T_p^* \mcal{M}$ is called a \emph{covector}.  
Let $\phi(p) = (\theta^1,\theta^2,\cdots,\theta^n)$ be the coordinates of $p$, it can be shown that the set of operators $\{\frac{\partial}{\partial \theta^1},\cdots,\frac{\partial}{\partial \theta^n}\}$ forms a basis for $\mcal{T}_p\mcal{M}$ and is called the \emph{coordinate basis}. Similarly, the dual space admits the \emph{dual coordinate basis} $\{\ud\theta^1,\cdots,\ud\theta^n\}$. These two sets of bases satisfy $\ud\theta^\mu(\partial_\nu) = \delta^\mu_\nu$ where $\delta^{\mu}_\nu \triangleq \begin{cases}
1, \mu = \nu\\
0, \mu \neq \nu
\end{cases}$ is the Kronecker delta. Note that in this paper we abbreviate $\frac{\partial}{\partial \theta^\mu}$ to $\partial_\mu$ and often refer to an entity (\eg, vector, covector and point on the manifold) with its coordinates.

Vectors and covectors are geometric objects associated with a manifold, which exist independently of the coordinate system. However, we rely on their \textbf{representations} \wrt some coordinate system for quantitative studies. Given a coordinate system, a vector $\mbf{a}$ (covector $\mbf{a}^*$) can be represented by its coefficients \wrt the coordinate (dual coordinate) bases, which we denote as $a^\mu$ ($a_\mu$). Therefore, these coefficients depend on a specific coordinate system, and will change for different parameterizations. In order for those coefficients to represent coordinate-independent entities like vectors and covectors, their change should obey some appropriate transformation rules. Let the new coordinate system under a different parameterization be $\phi'(p) = (\xi^1,\cdots,\xi^n)$ and let the old one be $\phi(p) = (\theta^1,\cdots,\theta^n)$. It can be shown that the new coefficients of $\mbf{a} \in \mcal{T}_p\mcal{M}$ will be given by 
$
 a^{\mu'} = a^\mu \frac{\partial \xi^{\mu'}}{\partial \theta^\mu},
$
while the new coefficients of $\mbf{a}^* \in \mcal{T}_p^*\mcal{M}$ will be determined by 
$
a_{\mu'} = a_\mu \frac{\partial \theta^\mu}{\partial \xi^{\mu'}}.
$
Due to the difference of transformation rules, we say $a^\mu$ is \emph{contravariant} while $a_\mu$ is \emph{covariant}, as indicated by superscripts and subscripts respectively. In this paper, we only use Greek letters to denote contravariant / covariant components.

Riemannian manifolds are equipped with a positive definite metric tensor $g_p \in \mcal{T}_p^* \mcal{M} \otimes \mcal{T}_p^* \mcal{M}$, so that distances and angles can be characterized. The inner product of two vectors $\mbf{a} = a^\mu \partial_\mu \in \mcal{T}_p\mcal{M}$, $\mbf{b} = b^\nu \partial_\nu \in \mcal{T}_p\mcal{M}$ is defined as $\langle \mbf{a}, \mbf{b} \rangle \triangleq g_p(\mbf{a}, \mbf{b}) = g_{\mu\nu} \ud\theta^\mu \otimes \ud\theta^\nu (a^\mu \partial_\mu, b^\nu\partial_\nu) = g_{\mu\nu} a^\mu b^\nu$. For convenience, we denote the inverse of the metric tensor as $g^{\alpha\beta}$ using superscripts, \ie, $g^{\alpha\beta} g_{\beta\mu} = \delta^\alpha_\mu$. The introduction of inner product induces a natural map from a tangent space to its dual space. Let $\mbf{a} = a^\mu \partial_\mu \in \mcal{T}_p\mcal{M}$, its natural correspondence in $\mcal{T}_p^*\mcal{M}$ is the covector $\mbf{a}^* \triangleq \langle \mbf{a}, \cdot \rangle = a_\nu d\theta^\nu$. It can be shown that $a_\nu = a^\mu g_{\mu\nu}$ and $a^\mu = g^{\mu\nu}a_\nu$. We say the metric tensor relates the coefficients of a vector and its covector by lowering and raising indices, which effectively changes the transformation rule.

The metric structure makes it possible to define \emph{geodesics} on the manifold, which are constant speed curves $\gamma: \mbb{R} \rightarrow \mcal{M}$ that are locally distance minimizing. Since the distances on manifolds are independent of parameterization, geodesics are invariant objects. Using a specific coordinate system, $\gamma(t)$ can be determined by solving the \emph{geodesic equation}
\begin{align}
\ddot \gamma^\mu + \Gamma_{\alpha\beta}^\mu \dot\gamma^\alpha \dot \gamma^\beta = 0\label{eq:geo},
\end{align}
where $\Gamma_{\alpha\beta}^\mu$ is the \emph{Levi-Civita connection} defined by
\begin{align}
\Gamma_{\alpha\beta}^\mu \triangleq \frac{1}{2} g^{\mu\nu}(\partial_\alpha g_{\nu\beta} + \partial_\beta g_{\nu\alpha} - \partial_\nu g_{\alpha\beta}).\label{eqn:connection}
\end{align}
Note that we use $\dot{\gamma}$ to denote $\frac{\ud \gamma}{\ud t}$ and $\ddot{\gamma}$ for $\frac{\ud^2 \gamma}{\ud t^2}$.

Given $p\in\mcal{M}$ and $v \in \mcal{T}_p\mcal{M}$, there exists a unique geodesic satisfying $\gamma(0)=p, \dot{\gamma}(0) = v$. If we follow the curve $\gamma(t)$ from $p = \gamma(0)$ for a unit time $\Delta t = 1$, we can reach another point $p' = \gamma(1)$ on the manifold. In this way, traveling along geodesics defines a map from $\mcal{M} \times \mcal{T}\mcal{M}$ to $\mcal{M}$ called \emph{exponential map}
\begin{align}
\operatorname{Exp}(p,v) \triangleq \gamma(1),
\end{align}
where $\gamma(0) = p$ and $\dot{\gamma}(0) = v$. By simple re-scaling we also have $\operatorname{Exp}(p, h v) = \gamma(h)$.

As a summary, we provide a graphical illustration of relevant concepts in Riemannian geometry in~\figref{fig:geometry}. 
Here we emphasize again the important ontological difference between an object and its coordinate. The manifold itself, along with geodesics and vectors (covectors) in its tangent (cotangent) spaces is intrinsic and independent of coordinates. The coordinates need to be transformed correctly to describe the same objects and properties on the manifold under a different coordinate system. This is where invariance emerges.

\subsection{Numerical Differential Equation Solvers}
Let the ordinary differential equation (ODE) be
$\dot{x}(t) = f(t, x(t))$, where $x(0) = a$ and $t \in [0, T]$. Numerical integrators try to trace $x(t)$ with iterative local approximations $\{x_k \mid k\in\mbb{N}\}$.

We discuss several useful numerical methods in this paper. The forward Euler method updates its approximation by $x_{k+1} = x_k + h f(t_k, x_k)$ and $t_{k+1} = t_{k} + h$. It can be shown that as $h\rightarrow 0$, the error $\norm{x_k - x(t_k)}$ can be bounded by $\mcal{O}(h)$. The midpoint integrator is a Runge-Kutta method with $\mcal{O}(h^2)$ error. Its update formula is given by $x_{k+1} = x_k + h f\left( t_k + \frac{1}{2} h ,x_k + \frac{h}{2}f(t_k,x_k) \right)$, $t_{k+1} = t_{k} + h$. The Riemannian Euler method (see pp.3-6 in~\cite{Bielecki2002-ix}) is a less common variant of the Euler method, which uses the Exponential map for its updates as $x_{k+1} = \operatorname{Exp}(x_k,h f(t_k, x_k))$, $t_{k+1} = t_{k} + h$. While having the same asymptotic error $\mcal{O}(h)$ as forward Euler, it has more desirable invariance properties.

\subsection{Revisiting Natural Gradient Method}
Let $r_\theta(\mbf{x}, \mbf{t}) = p_\theta(\mbf{t}\mid\mbf{x})q(\mbf{x})$ denote a probabilistic model parameterized by $\theta \in \Theta$, where $\mbf{x},\mbf{t}$ are random variables, $q(\mbf{x})$ is the marginal distribution of $\mbf{x}$ and assumed to be fixed. Conventionally, $\mbf{x}$ is used to denote the input and $\mbf{t}$ represents its label. In a differential geometric framework, the set of all possible probabilistic models $r_\theta$ constitutes a manifold $\mcal{M}$, and the parameter vector $\theta$ provides a coordinate system. Furthermore, the infinitesimal distance of probabilistic models can be measured by the Fisher information metric $g_{\mu\nu} = \mbb{E}_{\mbf{x}\sim q} \mbb{E}_{p_\theta(\mbf{t}\mid\mbf{x})}[\partial_\mu\log p_\theta(\mbf{t}\mid\mbf{x})\partial_\nu \log p_\theta(\mbf{t}\mid\mbf{x})]$. Let the loss function $L(r_\theta) = - \mbb{E}_{\mbf{x}\sim q} [\log p_\theta(\mbf{l}\mid\mbf{x})]$ be the expected negative log-likelihood, where $\mbf{l}$ denotes the ground truth labels in the training dataset. Our learning goal is to find a model $r_{\theta^*}$ that minimizes the (empirical) loss $L(r_\theta)$.

The well known update rule of gradient descent $\theta_{k+1}^\mu = \theta_k^\mu - h \lambda \partial_\mu L(r_{\theta_k})$ can be viewed as approximately solving the (continuous time) ODE
\begin{equation}
\dot\theta^\mu = -\lambda \partial_\mu L(r_\theta) \label{eqn:sgd}
\end{equation}
with forward Euler method. Here $\lambda$ is a time scale constant, $h$ is the step size, and their product $h\lambda$ is the learning rate. Note that $\lambda$ will only affect the ``speed'' but not trajectory of the system. It is notorious that the gradient descent ODE is not invariant to reparameterizations~\citep{ollivier2013riemannian,martens2010deep}. For example, if we rescale $\theta^\mu$ to $2 \theta^\mu$, $\partial_\mu L(r_\theta)$ will be downscaled to $\frac{1}{2} \partial_\mu L(r_\theta)$. This is more evident from a differential geometric point of view. As can be verified by chain rule, $\dot\theta^\mu$ transforms contravariantly and can therefore be treated as a vector in $\mcal{T}_p\mcal{M}$, while $\partial_\mu L(r_\theta)$ transforms covariantly, thus being a covector in $\mcal{T}^*_p\mcal{M}$. Because Eq.~\eqref{eqn:sgd} tries to relate objects in different spaces with different transformation rules, it is not an invariant relation.

Natural gradient alleviates this issue by approximately solving an invariant ODE. Recall that we can raise or lower an index given a metric tensor $g_{\mu\nu}$. By raising the index of $\partial_\mu L(r_\theta)$, the r.h.s. of the gradient descent ODE (Eq.~\eqref{eqn:sgd}) becomes a vector in $\mcal{T}_p\mcal{M}$, which solves the type mismatch problem of Eq.~\eqref{eqn:sgd}. The new ODE
\begin{equation}
\dot\theta^\mu = - \lambda g^{\mu\nu} \partial_\nu L(r_\theta) \label{eqn:ng}
\end{equation}
is now invariant, and the forward Euler approximation becomes $\theta_{k+1}^\mu = \theta_k^\mu - h \lambda g^{\mu\nu}\partial_\nu L(r_{\theta_k})$, which is the traditional natural gradient update~\citep{amari1998natural}.

\section{Higher-order Integrators}


If we could integrate the learning trajectory equation $\dot \theta^\mu = - \lambda g^{\mu\nu}\partial_\nu L$ exactly, the optimization procedure would be invariant to reparameterizations. 
However, the na\"{i}ve linear update of natural gradient $\theta_{k+1}^\mu = \theta_k^\mu - h\lambda g^{\mu\nu}\partial_\nu L$ is only a forward Euler approximation, 
and can only converge to the invariant exact solution in first order.
Therefore, a natural improvement is to use higher-order integrators to obtain a more accurate approximation to the exact solution.

As mentioned before, the midpoint integrator has second-order convergence and should be generally more accurate. In our case, it becomes
\begin{align*}
\theta_{k+\frac{1}{2}}^\mu &= \theta_k^\mu
 - \frac{1}{2}h\lambda g^{\mu\nu}(\theta_k)\partial_\nu L(r_{\theta_k}),\\
\theta_{k+1}^\mu &= \theta_k^\mu - h\lambda g^{\mu\nu}(\theta_{k+\frac{1}{2}}) \partial_\nu L(r_{\theta_{k+\frac{1}{2}}}).
\end{align*}
where $g^{\mu\nu}(\theta_k), g^{\mu\nu}(\theta_{k+\frac{1}{2}})$ are the inverse metrics evaluated at $\theta_k$ and $\theta_{k+\frac{1}{2}}$ respectively. Since our midpoint integrator converges to the invariant natural gradient ODE solution in second order, it preserves higher-order invariance compared to the first-order Euler integrator used in vanilla natural gradient.




\section{Riemannian Euler Method}
For solving the natural gradient ODE (Eq.~\eqref{eqn:ng}), the Riemannian Euler method's update rule becomes
\begin{align}
	\theta_{k+1}^\mu = \operatorname{Exp}(\theta_{k}^\mu, -h\lambda g^{\mu\nu}\partial_\nu L(r_{\theta_k})), \label{eq:reuler}
\end{align}
where $\operatorname{Exp}:\{ (p,v)\mid p\in \mcal{M}, v\in \mcal{T}_p\mcal{M} \}\rightarrow \mcal{M}$ is the exponential map as defined in \secref{sec:pre}. The solution obtained by Riemannian Euler method is invariant to reparameterizations, because $\operatorname{Exp}$ is a function independent of parameterization and for each step, the two arguments of $\operatorname{Exp}$ are both invariant. 


\subsection{Geodesic Correction}

For most models, it is not tractable to compute $\operatorname{Exp}$, since it requires solving the geodesic equation \eqref{eq:geo} exactly. Nonetheless, there are two numerical methods to approximate geodesics, with different levels of accuracy. 

According to \secref{sec:pre}, 
$\operatorname{Exp}(\theta_k^\mu, -h \lambda  g^{\mu\nu}\partial_\nu L(r_{\theta_k})) = \gamma^\mu_k(h)$, where $\gamma^\mu_k$ satisfies the geodesic equation \eqref{eq:geo} and
\begin{align*}
\gamma^\mu_k(0) &= \theta_k^\mu \\
\dot\gamma^\mu_k(0) &= - \lambda g^{\mu\nu} \partial_\nu L(r_{\theta_k}).
\end{align*}
The first method for approximately solving $\gamma^\mu_k(t)$ ignores the whole geodesic equation and only uses information of first derivatives, giving $$\gamma_k^\mu(h) \approx \theta_{k}^\mu + h \dot{\gamma}^\mu_k(0) = \theta_k^\mu - h \lambda g^{\mu\nu}\partial_\nu L,$$which corresponds to the na\"{i}ve natural gradient update rule. 

The more accurate method leverages information of second derivatives from the geodesic equation~\eqref{eq:geo}. The result is 
\begin{align*}
\gamma^\mu_k(h) &\approx \theta^\mu_k + h \dot{\gamma}_k^\mu(0) + \frac{1}{2}h^2 \ddot{\gamma}^\mu_k(0)\\
 &= \theta^\mu_k - h \lambda g^{\mu\nu}\partial_\nu L - \frac{1}{2}h^2 \Gamma^{\mu}_{\alpha\beta} \dot{\gamma}_k^\alpha(0) \dot{\gamma}_k^\beta(0).
\end{align*}
The additional second-order term given by the geodesic equation (\ref{eq:geo})  reduces the truncation error to third-order. This corresponds to our new natural gradient update rule with \emph{geodesic correction}, \ie, 
\begin{align}
\theta_{k+1}^\mu = \theta_{k}^\mu + h \dot{\gamma}^\mu_k(0) - \frac{1}{2}h^2 \Gamma^\mu_{\alpha\beta} \dot{\gamma}^\alpha_k(0) \dot{\gamma}^\beta_k(0),\label{main:geo:eq}
\end{align}
where $\dot{\gamma}^\mu_k(0) = -\lambda g^{\mu\nu}\partial_\nu L(r_{\theta_k})$.

\subsection{Faster Geodesic Correction}

To obtain the second order term in the geodesic corrected update, we first need to compute $\dot{\gamma}_k(0)$, which requires inverting the Fisher information matrix. Then we have to plug in $\dot{\gamma}_k(0)$ and compute $\Gamma^\mu_{\alpha\beta}\dot{\gamma}_k^\alpha(0) \dot{\gamma}_k^\beta(0)$, which involves inverting the same Fisher information matrix again (see \eqref{eqn:connection}). Matrix inversion (more precisely, solving the corresponding linear system) is expensive and it would be beneficial to combine the natural gradient and geodesic correction terms together and do only one inversion.

To this end, we propose to estimate $\dot\gamma_k(0)$ in $\Gamma_{\alpha\beta}^\mu \dot{\gamma}_k(0)^\alpha \dot{\gamma}_k(0)^\beta$ with $\dot\gamma_k(0) \approx (\theta_k - \theta_{k - 1})/h$. 
Using this approximation and substituting (\ref{eqn:connection}) into (\ref{main:geo:eq}) gives the following faster geodesic correction update rule:
\begin{align}
\delta\theta_{k}^\mu &= \lambda g^{\mu\nu}\cdot
\bigg[ -\partial_\nu L(r_{\theta_k}) - 
\\ &\frac{1}{4} h\lambda (\partial_\alpha g_{\nu\beta} + \partial_\beta g_{\nu\alpha} - \partial_\nu g_{\alpha\beta})\delta\theta_{k-1}^\alpha\delta\theta_{k-1}^\beta\bigg] \nonumber\\
\theta_{k+1}^\mu &= \theta_{k}^\mu +h\delta\theta_{k}^\mu,\label{main:fastergeo:eq}
\end{align}
which only involves one inversion of the Fisher information matrix. 

\subsection{Convergence Theorem}
We summarize the convergence properties of geodesic correction and its faster variant in the following general theorem.
\begin{theorem}[Informal]\label{thm:converge}
	Consider the initial value problem $\dot x = f(t, x(t)), x(0) = a, 0\leq t \leq T$. Let the interval $[0, T]$ be subdivided into $n$ equal parts by the grid points $0 = t_0 < t_1 < \cdots < t_n = T$, with the grid size $h = T / n$. Denote $x_k$ and $\hat{x}_k$ as the numerical solution given by geodesic correction and its faster version respectively.	Define the error $e_k$ at each grid point $x_k$ by $e_k = x_k' - x_k$, and $\hat{e}_k = x_k' - \hat{x}_k$, where $x_k'$ is the numerical solution given by Riemannian Euler method. Then it follows that
	\begin{equation*}
	\norm{e_k} \leq \mcal{O}(h^2)\quad \text{and} \quad	\norm{\hat{e}_k} \leq \mcal{O}(h^2),  h \rightarrow 0,  \forall k \in [n].
	\end{equation*}
	As a corollary, both Euler's update with geodesic correction and its faster variant converge to the solution of ODE in 1st order.
\end{theorem}
\begin{proof}
	Please refer to Appendix~\ref{app:proofs} for a rigorous statement and detailed proof.
\end{proof}

The statement of \thmref{thm:converge} is general enough to hold beyond the natural gradient ODE (Eq.~\eqref{eqn:ng}). It shows that both geodesic correction and its faster variant converge to the invariant Riemannian Euler method in 2nd order. In contrast, vanilla forward Euler method, as used in traditional natural gradient, is a first order approximation of Riemannian Euler method. In this sense, geodesic corrected updates preserve higher-order invariance.

\section{Geodesic Correction for Neural Networks}
Adding geodesic correction requires computing the Levi-Civita connection $\Gamma_{\mu\nu}^\alpha$ (see \eqref{eqn:connection}), which usually involves second-order derivatives. This is to the contrast of natural gradient, where the computation of Fisher information matrix only involves first-order derivatives of outputs. In this section, we address the computational issues of geodesic correction in optimizing deep neural networks.


In order to use natural gradient for neural network training, we first need to convert neural networks to probabilistic models. A feed-forward network can be treated as a conditional distribution $p_\theta(\mbf{t}\mid \mbf{x})$. For regression networks, $p_\theta(\mbf{t}\mid \mbf{x})$ is usually a family of multivariate Gaussians. For classification networks, $p_\theta(\mbf{t}\mid\mbf{x})$ usually becomes a family of categorical distributions. The joint probability density is $q(\mbf{x})p_\theta(\mbf{t}\mid\mbf{x})$, where $q(\mbf{x})$ is the data distribution and is usually approximated with the empirical distribution.

The first result in this section is the analytical formula of the Levi-Civita connection of neural networks.

\begin{proposition}
\label{analytic-levi}
	The Levi-Civita connection of a neural network model manifold is given by
	\begin{multline}
	\Gamma^\mu_{\alpha\beta} = g^{\mu\nu} \mbb{E}_{q(\mbf{x})} \mbb{E}_{p_\theta(\mbf{t}\mid\mbf{x})} \\\bigg\{ \partial_\nu \log p_\theta(\mbf{t}\mid\mbf{x})
	\bigg[ \partial_\alpha \partial_\beta \log p_\theta(\mbf{t}\mid\mbf{x}) +\\ \frac{1}{2}\partial_\alpha \log p_\theta(\mbf{t}\mid\mbf{x}) \partial_\beta \log p_\theta(\mbf{t}\mid\mbf{x})\bigg]\bigg\}
	\end{multline}
\end{proposition}
\begin{proof}
In Appendix~\ref{app:proofs}.
\end{proof}
	%

We denote the outputs of a neural network as $\mbf{y}(\mbf{x}, \theta) = (y_1,y_2,\cdots,y_o)$, which is an $o$-dimensional vector if there are $o$ output units. In this paper, we assume that $\mbf{y}(\mbf{x},\theta)$ are the values \emph{after} final layer activation (\eg, softmax). 
For typical loss functions, the expectation with respect to the corresponding distributions can be calculated analytically. Specifically, we instantiate the Levi-Civita connection for model distributions induced by three common losses and summarize them in the following proposition.

\begin{proposition}\label{prop:l2}
For the squared loss, we have
\begin{align*}
p_\theta(\mbf{t}\mid\mbf{x}) &= \prod_{i=1}^o \mcal{N}(t_i \mid y_i, \sigma^2)\\
g_{\mu\nu} &= \frac{1}{\sigma^2} \sum_{i=1}^o \mbb{E}_{q(\mbf{x})}[\partial_\mu y_i \partial_\nu y_i]\\
\Gamma^\mu_{\alpha\beta} &= \frac{1}{\sigma^2} \sum_{i=1}^o g^{\mu\nu} \mbb{E}_{q(\mbf{x})} [\partial_\nu y_i \partial_\alpha \partial_\beta y_i]
\end{align*}
For the binary cross-entropy loss, we have 
\begin{align*}
p_\theta(\mbf{t}\mid \mbf{x}) &= \prod_{i=1}^o y_i^{t_i} (1-y_i)^{1-t_i}\\
g_{\mu\nu} &= \sum_{i=1}^{o} ~\mbb{E}_{q(\mbf{x})} \bigg[~\frac{1}{y_i(1-y_i)}\cdot\partial_\mu y_i \partial_\nu y_i~\bigg]\\
\Gamma^\mu_{\alpha\beta} &= g^{\mu\nu} \sum_{i=1}^o~ \mbb{E}_{q(\mbf{x})}\bigg[ ~\frac{2y_i-1}{2y_i^2(1-y_i)^2} \cdot \partial_\nu y_i \partial_\alpha y_i \partial_\beta y_i \\
&\qquad\qquad\qquad+ \frac{1}{y_i(1-y_i)}\cdot \partial_\nu y_i\partial_\alpha\partial_\beta y_i ~\bigg].
\end{align*}
In the case of multi-class cross-entropy loss, we have
\begin{align*}
p_\theta(\mbf{t}\mid\mbf{x}) &= \prod_{i=1}^o y_i^{t_i}\\
g_{\mu\nu} &= \frac{1}{\sigma^2} \sum_{i=1}^o~ \mbb{E}_{q(\mbf{x})}\bigg[\frac{1}{y_i}\cdot \partial_\mu y_i \partial_\nu y_i\bigg]\\
\Gamma^\mu_{\alpha\beta} &= g^{\mu\nu} \sum_{i=1}^o ~\mbb{E}_{q(\mbf{x})}\bigg[ \frac{1}{y_i} \cdot \partial_\nu y_i \partial_\alpha \partial_\beta y_i \\
&\qquad\qquad\qquad- \frac{1}{2y_i^2}\cdot \partial_\nu y_i \partial_\alpha y_i \partial_\beta y_i \bigg].
\end{align*}
\end{proposition}
\begin{proof}
	In Appendix~\ref{app:losses}.
\end{proof}
For geodesic correction, we only need to compute connection-vector products $\Gamma_{\alpha\beta}^\mu \dot{\gamma}^\alpha \dot{\gamma}^\beta$. This can be done with a similar idea to Hessian-vector products~\citep{pearlmutter1994fast}, for which we provide detailed derivations and pseudocodes in Appendix~\ref{app:computation}. It can also be easily handled with automatic differentiation frameworks. We discuss some practical considerations on how to apply them in real cases in Appendix~\ref{app:settings}.
\begin{figure*}
	\centering
	\includegraphics[width=0.95\textwidth]{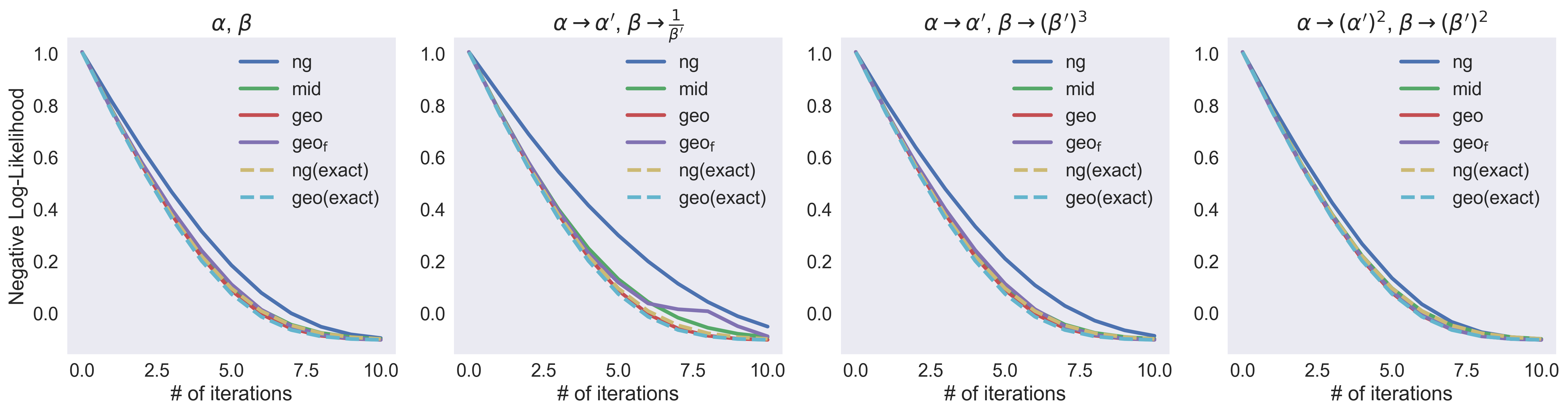}
	\caption{The effect of re-parameterizations on algorithms fitting a univariate Gamma distribution. Titles indicate which parameterization was used.}\label{fig:invariance}
\end{figure*}
\section{Related Work}\label{sec:related}

The idea of using the geodesic equation to accelerate gradient descent on manifolds was first introduced in~\citet{transtrum2011geometry}. However, our geodesic correction has several important differences. Our framework is generally applicable to all probabilistic models. This is to be contrasted with ``geodesic acceleration'' in ~\citet{transtrum2011geometry} and \citet{transtrum2012geodesic}, which can only be applied to nonlinear least squares. Additionally, our geodesic correction is motivated from the perspective of preserving higher-order invariance, while in~\citet{transtrum2012geodesic} it is motivated as a higher-order correction to the Gaussian-Newton approximation of the Hessian under the so-called ``small-curvature assumption''. We discuss and evaluate empirically in Appendix~\ref{app:small} why the small-curvature approximation does not hold for training deep neural networks.

There has been a resurgence of interest in applying natural gradient to neural network training. \citet{martens2010deep}{ and \citet{martens2011learning} show that Hessian-Free optimization, which is equivalent to natural gradient method in important cases in practice~\citep{pascanu2013revisiting,martens2014new}, is able to obtain state-of-the-art results in optimizing deep autoencoders and RNNs. To scale up natural gradient, some approximations for inverting the Fisher information matrix have been recently proposed, such as Krylov subspace descent~\citep{vinyals2012krylov}, FANG~\citep{grosse2015scaling} and K-FAC~\citep{martens2015optimizing,grosse2016kronecker,badistributed}.
\section{Experimental Evaluations}
In this section, we demonstrate the benefit of respecting higher-order invariance through experiments on synthetic optimization problems, deep neural net optimization tasks and policy optimization in deep reinforcement learning.

Algorithms have abbreviated names in figures. We use ``\textbf{ng}'' to denote the basic natural gradient, ``\textbf{geo}'' to denote the one with geodesic correction, ``\textbf{geo\textsubscript{f}}'' to denote the faster geodesic correction, and ``\textbf{mid}'' to abbreviate natural gradient update using midpoint integrator.

\subsection{Invariance}
In this experiment, we investigate the effect of invariance under different parameterizations of the same objective. We test different algorithms on fitting a univariate Gamma distribution via Maximum Log-Likelihood. The problem is simple---we can calculate the Fisher information metric and corresponding Levi-Civita connection accurately. Moreover, we can use ODE-solving software to numerically integrate the continuous natural gradient equation and calculate the exponential map used in Riemannian Euler method.

The pdf of Gamma distribution is 
\begin{equation*}
	p(x \mid \alpha, \beta) = \Gamma(x; \alpha, \beta) \triangleq \frac{\beta^\alpha}{\Gamma(\alpha)} x^{\alpha - 1} e^{-\beta x},
\end{equation*}
where $\alpha$, $\beta$ are shape and rate parameters. Aside from the original parameterization, we test three others: 1) $\alpha = \alpha', \beta = 1/\beta'$; 2) $\alpha = \alpha', \beta = (\beta')^3$ and 3) $\alpha = (\alpha')^2, \beta = (\beta')^2$, where $\alpha'$, $\beta'$ are new parameters. We generate 10000 synthetic data points from $\Gamma(X; 20, 20)$. During training, $\alpha$ and $\beta$ are initialized to 1 and the learning rate is fixed to 0.5. 

We summarize the results in \figref{fig:invariance}. Here ``\textbf{ng(exact)}'' is obtained by numerically integrating \eqref{eqn:ng}, and ``\textbf{geo(exact)}'' is obtained using Riemannian Euler method with a numerically calculated exponential map function. As predicted by the theory, both methods are exactly invariant under all parameterizations. From \figref{fig:invariance} we observe that the vanilla natural gradient update is not invariant under re-parameterizations, due to its finite step size. We observe that our midpoint natural gradient method and geodesic corrected algorithms are more resilient to re-parameterizations, and all lead to accelerated convergence of natural gradient.

\begin{figure*}
	\centering
	\includegraphics[width=0.8\textwidth]{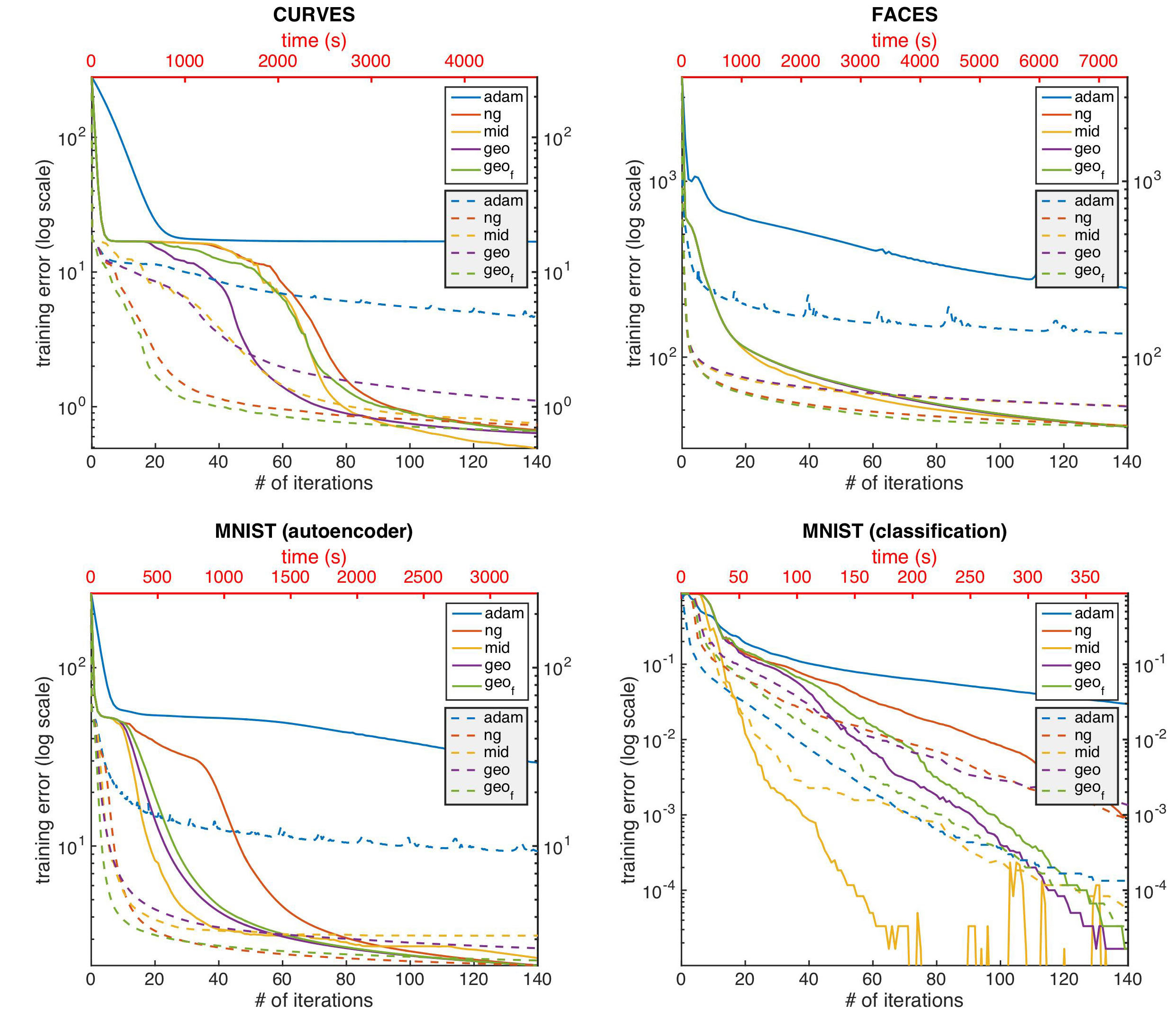}
	\caption{Training deep auto-encoders and classifiers with different acceleration algorithms. Solid lines show performance against number of iterations (bottom axes) while dashed lines depict performance against running time (top axes).} 	\label{fig:deepencoder}	\label{fig:time}
\end{figure*}


\begin{figure*}
\centering
\includegraphics[width=0.9\textwidth]{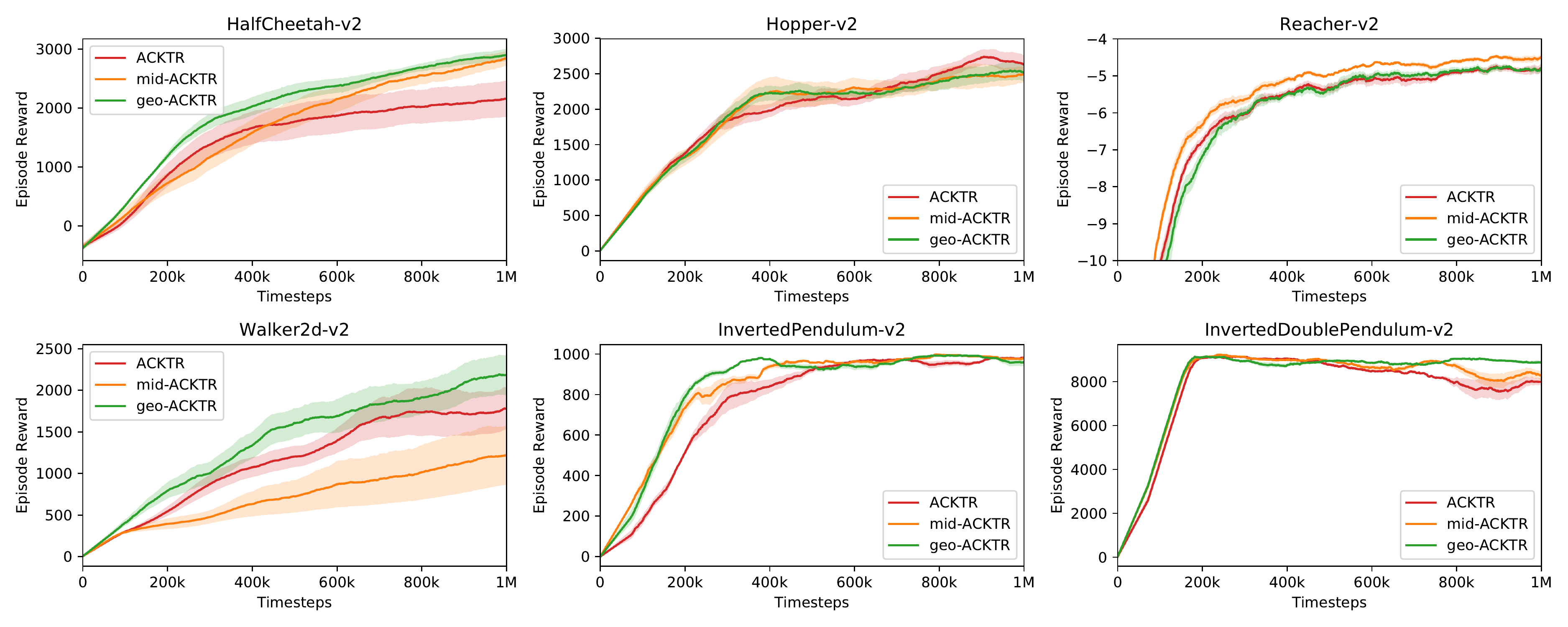}
\caption{Sample efficiency of model-free reinforcement learning on continuous control tasks~\citep{todorov2012mujoco}. Titles indicate the environment used in OpenAI Gym~\citep{brockman2016openai}.}
\label{fig:acktr}
\end{figure*}

\subsection{Training Deep Neural Nets}
We test our algorithms on deep autoencoding and classification problems. The datasets are CURVES, MNIST and FACES, all of which contain small gray-scale images of various objects, \ie, synthetic curves, hand-written digits and human faces. Since all deep networks use fully-connected layers and sigmoid activation functions, the tasks are non-trivial to solve even for modern deep learning optimizers, such as Adam~\cite{kingma2014adam}. Due to the high difficulty of this task, it has become a standard benchmark for neural network optimization algorithms~\citep{hinton2006reducing,martens2010deep,vinyals2012krylov,sutskever2013importance,martens2015optimizing}. Since these tasks only involve squared loss and binary cross-entropy, we additionally test multi-class cross-entropy on a MNIST classification task. Additional details can be found in Appendix~\ref{app:exp1}.

\figref{fig:deepencoder} summarizes our results on all three datasets. Here ``\textbf{ng}'' is the natural gradient method (Hessian-Free) as implemented in~\citet{martens2010deep}. For completeness, we also add Adam~\citep{kingma2014adam} into comparison and denote it as ``\textbf{adam}''. The training error reported for all datasets is the squared reconstruction error. Since the test error traces the training error very well, we only report the result of training error due to space limitation. It is clear that all acceleration methods lead to per iteration improvements compared to na\"{i}ve natural gradient. 
It is also remarkable that the performance of ``\textbf{geo\textsubscript{f}}'', while being roughly half as expensive as ``\textbf{geo}'' per iteration, does not degrade too much compared to ``\textbf{geo}''. For performance comparisons with respect to time, ``\textbf{geo\textsubscript{f}}'' is usually the best (or comparable to the best). ``\textbf{mid}'' and ``\textbf{geo}'' are relatively slower, since they need roughly twice as much computation per iteration as ``\textbf{ng}''. Nonetheless, ``\textbf{mid}'' still has the best time performance for MNIST classification task. 

We hereby emphasize again that geodesic correction methods are not aimed for providing more accurate solutions of the natural gradient ODE. Instead, they are higher-order approximations of an invariant solution (obtained by Riemannian Euler method), which itself is a first-order approximation to the exact solution. The improvements of both geodesic correction and midpoint integrator in \figref{fig:invariance} and \figref{fig:deepencoder} confirm our intuition that preserving higher-order invariance can accelerate natural gradient optimization.

\subsection{Model-free Reinforcement Learning for Continuous Control}
Finally, we evaluate our methods in reinforcement learning over six continuous control tasks~\citep{todorov2012mujoco}. Specifically, we consider improving the algorithm of ACKTR~\citep{wu2017scalable}, an efficient variant of natural policy gradient~\citep{kakade2002natural} which uses Kronecker factors~\citep{martens2015optimizing} to approximately compute the inverse of Fisher information matrix. For these methods, we evaluate sample efficiency (expected rewards per episode reached within certain numbers of interactions); in robotics tasks the cost of simulation often dominates the cost of the reinforcement learning algorithm, so requiring less interactions to achieve certain performance has higher priority than lower optimization time per iteration. Therefore we only test midpoint integrator and geodesic correction method for improving ACKTR, and omit the faster geodesic correction because of its less accurate approximation.

Figure~\ref{fig:acktr} describes the results on the continuous control tasks, where we use ``\textbf{mid-}'' and ``\textbf{geo-}'' to denote our midpoint integrator and geodesic correction methods for ACKTR respectively. In each environment, we consider the same constant learning rate schedule for all three methods (detailed settings in Appendix~\ref{app:exp2}). While the Fisher information matrices are approximated via Kronecker factors, our midpoint integrator and geodesic correction methods are still able to outperform ACKTR in terms of sample efficiency in most of the environments. This suggests that preserving higher-order invariance could also benefit natural policy gradients in reinforcement learning, and our methods can be scaled to large problems via approximations of Fisher information matrices.
\section{Conclusion}
Our contributions in this paper can be summarized as:

\begin{itemize}[leftmargin=*]
	\item We propose to measure the invariance of numerical schemes by comparing their convergence to idealized invariant solutions.
	\item To the best of our knowledge, we are the first to use midpoint integrators for natural gradient optimization.
	\item Based on Riemannian Euler method, we introduce geodesic corrected updates. Moreover, the faster geodesic correction has comparable time complexity with vanilla natural gradient. Computationally, we also introduce new backpropagation type algorithms to compute connection-vector products. Theoretically, we provide convergence proofs for both types of geodesic corrected updates.
	\item Experiments confirm the benefits of invariance and demonstrate faster convergence and improved sample efficiency of our proposed algorithms in supervised learning and reinforcement learning applications.
\end{itemize}

For future research, it would be interesting to perform a thorough investigation over applications in reinforcement learning, and studying faster variants and more efficient implementations of the proposed acceleration algorithms.

\subsection*{Acknowledgements}
The authors would like to thank Jonathan McKinney for helpful discussions. This work was supported by NSF grants \#1651565, \#1522054, \#1733686, Toyota Research Institute, Future of Life Institute, and Intel.




{
\bibliography{geo}}
\newpage
\appendix
\onecolumn
\section{Proofs}\label{app:proofs}

\begin{customprop}{1}
	The Levi-Civita connection of a neural network model manifold is given by
	\begin{align}
	\Gamma^\mu_{\alpha\beta} = g^{\mu\nu} \mbb{E}_{q(\mbf{x})} \mbb{E}_{p_\theta(\mbf{t}\mid\mbf{x})} \bigg\{ \partial_\nu \log p_\theta(\mbf{t}\mid\mbf{x})
	\bigg[ \partial_\alpha \partial_\beta \log p_\theta(\mbf{t}\mid\mbf{x}) + \frac{1}{2}\partial_\alpha \log p_\theta(\mbf{t}\mid\mbf{x}) \partial_\beta \log p_\theta(\mbf{t}\mid\mbf{x})\bigg]\bigg\}
	\nonumber
	\end{align}
\end{customprop}
\begin{proof}[Proof of \propref{analytic-levi}]
	Let $r_\theta(\mbf{z})$ be the joint distribution defined by $q(\mbf{x})p_\theta(\mbf{t}\mid\mbf{x})$. The Fisher information metric is $g_{\mu\nu} = \mbb{E}_{\mbf{z}} [\partial_\mu \log r_\theta(\mbf{z}) \partial_\nu \log r_\theta(\mbf{z})]$. The partial derivative $\partial_\mu g_{\alpha\beta}$ can be computed via scoring function trick, \ie, $\partial_\mu g_{\alpha\beta} = \partial_\mu \mbb{E}_{\mbf{z}}[\partial_\alpha \log r_\theta(\mbf{z})\partial_\beta \log r_\theta(\mbf{z})] = \mbb{E}_\mbf{z}[\partial_\mu \log r_\theta(\mbf{z}) \partial_\alpha \log r_\theta(\mbf{z}) \partial_\beta \log r_\theta(\mbf{z}) + \partial_\mu\partial_\alpha \log r_\theta(\mbf{z}) \partial_\beta \log r_\theta(\mbf{z})+\partial_\alpha \log r_\theta(\mbf{z}) \partial_\mu \partial_\beta \log r_\theta(\mbf{z})]$. All other partial derivatives can be obtained via symmetry.
	
	According to the definition, 
	\begin{align*}
	\Gamma^\mu_{\alpha\beta} &= \frac{1}{2} g^{\mu\nu} (\partial_\alpha g_{\nu\beta} + \partial_\beta g_{\nu\alpha} - \partial_\nu g_{\alpha\beta})
	\\ &= g^{\mu\nu}\mbb{E}_{\mbf{z}}\big\{ \partial_\nu \log r_\theta(\mbf{z}) \big[ \partial_\alpha\partial_\beta \log r_\theta(\mbf{z}) + \frac{1}{2} \partial_\alpha \log r_\theta(\mbf{z}) \partial_\beta \log r_\theta(\mbf{z})\big]\big\}.
	\end{align*}
	The proposition is proved after replacing $r_\theta(\mbf{z})$ with $q(\mbf{x})p_\theta(\mbf{t}\mid\mbf{x})$.
\end{proof}

\begin{customthm}{1}
	Consider the initial value problem $\dot x = f(t, x(t)), x(0) = a, 0\leq t \leq T$, where $f(t, x(t))$ is a continuous function in a region $U$ containing
	\begin{align*}
	\mcal{D} := \{ (t, x) \mid  0\leq t\leq T ,\norm{x-a} \leq X \}.
	\end{align*}
	Suppose $f(t, x)$ also satisfies Lipschitz condition, so that
	\begin{align*}
	\norm{f(t, x) - f(t, y)} \leq L_f\norm{x-y},\quad \forall (t, x)\in \mcal{D} \wedge (t,y) \in \mcal{D}.
	\end{align*}
	In addition, let $M = \sup_{(t,x) \in \mcal{D}} \norm{f(t,x)}$ and assume that $TM \leq X$.
	Next, suppose $x$ to be the global coordinate of a Riemannian $C^\infty$-manifold $\mcal{M}$ with Levi-Civita connection $\Gamma_{\alpha\beta}^{\mu}(x)$. Then, let the interval $[0, T]$ be subdivided into $n$ equal parts by the grid points $0 = t_0 < t_1 < \cdots < t_n = T$, with the grid size $h = T / n$. Denote $x_k$ as the numerical solution given by Euler's update with geodesic correction~\eqref{main:geo:eq}, so that
	\begin{align}
	x_{k+1}^\mu = x_k^\mu + h f^\mu(t_k, x_k) - \frac{1}{2} h^2 \Gamma^{\mu}_{\alpha\beta}(x_k)f^\alpha(t_k, x_k) f^\beta(t_k, x_k),\quad\quad x_0 = a.\label{eq:thm1_geoacc}
	\end{align}
	Let $\hat{x}_k$ be the numerical solution obtained by faster geodesic correction~\eqref{main:fastergeo:eq}
	\begin{align}
	\hat{x}_{k+1}^\mu = \hat{x}_{k}^\mu + hf^\mu(t_k, \hat{x}_k) - \frac{1}{2}h^2 \Gamma^\mu_{\alpha\beta}(\hat{x}_k)\delta \hat{x}_k^\alpha \delta \hat{x}_k^\beta,\label{eq:thm1_fastergeo}
	\end{align}
	where $\delta \hat{x}_k = (\hat{x}_k - \hat{x}_{k-1})/h$.
	Finally, define the error $e_k$ at each grid point $x_k$ by $e_k = x_k' - x_k$, and $\hat{e}_k = x_k' - \hat{x}_k$, where $x_k'$ is the numerical solution given by Riemannian Euler method, \ie,
	\begin{align}
	x_{k+1}' = \operatorname{Exp}(x_k', hf(t_k, x_k')),\quad\quad x_0' = a.\label{eq:thm1_reuler}
	\end{align}
	Then, assuming $\norm{x_k - a} \leq X$, $\norm{\hat{x}_k -a}\leq X$, and $\norm{x_k' - a}\leq X$, it follows that
	\begin{equation*}
	\norm{e_k} \leq \mcal{O}(h^2)\quad \text{and} \quad	\norm{\hat{e}_k} \leq \mcal{O}(h^2),\quad  h \rightarrow 0, \quad \forall k \in [n].
	\end{equation*}
	As a corollary, both Euler's update with geodesic correction and its faster variant converge to the solution of ODE in 1st order.
\end{customthm}

\begin{proof}[Proof of \thmref{thm:converge}]
		Since $f(t, x)$ is continuous in the compact region $\mcal{D}$ satisfying Lipschitz condition, and $TM\leq X$, Picard-Lindelöf theorem ensures that there exists a unique continuously differentiable solution in $\mcal{D}$.
		
		On the smooth Riemannian manifold $\mcal{M}$, the corresponding region $\mcal{X} \subset \mcal{M}$ for the set of coordinates $\{x \mid \norm{x-a}\leq X\}$ is compact. Following Lemma 5.7, Lemma 5.8 in~\cite{petersen2006riemannian} and a standard covering-of-compact-set argument, there exists a constant $\epsilon > 0$ such that geodesics are defined everywhere within the region $\mcal{G}:=  \{ (p,v, s)\mid p\in \mcal{X}, v\in\mcal{T}_p\mcal{M} \wedge \norm{v} \leq M, 0\leq s\leq \epsilon \}$. Tychonoff's theorem affirms that $\mcal{G}$ is compact in terms of product topology and we henceforth assume $h < \epsilon$. Next, let $\gamma(p,v,s) := \operatorname{Exp}(p,s v)$, where $s \in (-\epsilon, \epsilon)$. From continuous dependence on initial conditions of ODE, we have $\gamma(p,v,s) \in C^\infty$ (see, \eg, Theorem 5.11 in~\cite{petersen2006riemannian}).	
		
		Using Taylor expansion with Lagrange remainder, we can rewrite update rule \eqref{eq:thm1_reuler} as
		\begin{align}
		x'_{k+1} = x'_k + \frac{\partial \gamma(x_k', v_k', 0)}{\partial s} h + \frac{1}{2} \frac{\partial^2 \gamma(x_k', v_k', 0)}{\partial s^2}h^2 + \frac{1}{6} \frac{\partial^3 \gamma(x_k', v_k', \xi h)}{\partial s^3} h^3, \label{eq:thm1_1}
		\end{align}	
		where $v_k' = f(t_k, x_k')$ and $\xi \in [0,1]$. In the meanwhile, update rule \eqref{eq:thm1_geoacc} can be equivalently written as
		\begin{align}
		x_{k+1} = x_k + \frac{\partial \gamma(x_k, v_k, 0)}{\partial s}h + \frac{1}{2} \frac{\partial^2 \gamma(x_k, v_k, 0)}{\partial s^2} h^2,\label{eq:thm1_2}
		\end{align}
		where $v_k = f(t_k, x_k)$. Subtracting \eqref{eq:thm1_2} from \eqref{eq:thm1_1}, we obtain
		\begin{equation}
		\begin{aligned}
		e_{k+1} = e_k &+ \left[\frac{\partial \gamma(x_k', v_k', 0)}{\partial s} - \frac{\partial \gamma(x_k, v_k, 0)}{\partial s}\right] h \\
		&+ \frac{1}{2}\left[ \frac{\partial^2 \gamma(x_k', v_k', 0)}{\partial s^2} - \frac{\partial^2 \gamma(x_k, v_k, 0)}{\partial s^2} \right]h^2\\
		&+ \frac{1}{6} \frac{\partial^3 \gamma(x_k', v_k', \xi h)}{\partial s^3} h^3.
		\end{aligned}\label{eq:thm1_3}
		\end{equation}
				
		Recall that $\gamma(p,v,s)$ is $C^\infty$ on $\mcal{G}$, by extreme value theorem we can denote 
		\begin{align*}
		\Gamma_1 = \sup_{(p,v,s)\in\mcal{G}} \norm{\frac{\partial^2 \gamma(p,v,s)}{\partial p \partial s}},&\quad \Gamma_2 = \sup_{(p,v,s)\in\mcal{G}} \norm{\frac{\partial^2 \gamma(p,v,s)}{\partial v\partial s}}\\
		\Gamma_3 = \sup_{(p,v,s)\in\mcal{G}} \norm{\frac{\partial^3 \gamma(p,v,s)}{\partial p\partial s^2}},&\quad \Gamma_4 = \sup_{(p,v,s)\in\mcal{G}} \norm{\frac{\partial^3 \gamma(p,v,s)}{\partial v\partial s^2}}\\
		\Gamma_5 = \sup_{(p,v,s)\in\mcal{G}} &\norm{\frac{\partial^3 \gamma(p,v,s)}{\partial s^3}},
		\end{align*}
		where $\norm{\cdot}$ can be chosen arbitrarily as long as they are compatible (\eg, use operator norms for tensors), since norms in a finite dimensional space are equivalent to each other. 
		
		With upper bounds $\Gamma_1$, $\Gamma_2$, $\Gamma_3$, $\Gamma_4$ and $\Gamma_5$, Eq.~\eqref{eq:thm1_3} has the estimation via Lagrange mean value theorem,
		\begin{align*}
		\norm{e_{k+1}} &\leq \norm{e_k} + h (\Gamma_1 \norm{e_k} + \Gamma_2 \norm{v_k'-v_k})+ \frac{1}{2} h^2(\Gamma_3 \norm{e_k} + \Gamma_4 \norm{v_k'-v_k}) + \frac{1}{6}\Gamma_5 h^3\\
		&\leq \norm{e_k} + h (\Gamma_1 \norm{e_k} + L_f\Gamma_2 \norm{e_k})+ \frac{1}{2} h^2(\Gamma_3 \norm{e_k} + L_f\Gamma_4 \norm{e_k}) + \frac{1}{6}\Gamma_5 h^3\\
		&\leq \left( 1+\Gamma_1 h + \Gamma_2 h L_f + \frac{1}{2}\Gamma_3 h^2 + \frac{1}{2}\Gamma_4 L_f h^2 \right)\norm{e_k} + \frac{1}{6} \Gamma_5 h^3\\
		&\leq \cdots\\
		&\leq \left( 1+C_1 h+\frac{1}{2}C_2 h^2 \right)^k \left( \norm{e_0} + \frac{\Gamma_5 h^2}{6C_1 + 3C_2 h} \right) - \frac{\Gamma_5 h^2}{6C_1 + 3C_2 h}\\
		&\leq \left(e^{C_1 k h+\frac{1}{2}C_2 kh^2} - 1\right)\left(\frac{\Gamma_5 h^2}{6C_1 + 3C_2 h} \right)\\
		&\leq \left(e^{C_1 T+\frac{1}{2}C_2 T^2} - 1\right)\left(\frac{\Gamma_5 h^2}{6C_1 + 3C_2 h} \right)\\
		&\leq \frac{\Gamma_5}{6C_1}\left(e^{C_1 T+\frac{1}{2}C_2 T^2} - 1\right) h^2 = \mcal{O}(h^2),
		\end{align*}
		where from the fifth line we substitute $C_1$ for $\Gamma_1 + \Gamma_2 L_f$ and $C_2$ for $\Gamma_3 + \Gamma_4 L_f$, and we used the fact that $\norm{e_0}=0$ and the inequality $(1+w)^k \leq \exp(kw)$. This means geodesic correction converges to the invariant solution obtained using a Riemannian Euler method with a 2nd-order rate. Since Riemannian Euler method is itself a first-order algorithm~\citep{Bielecki2002-ix}, geodesic correction also converges to the exact solution in 1st order.
		
		Then, we consider the faster geodesic update rule~\eqref{eq:thm1_fastergeo}
		\begin{align}
		&\hat{x}_{k+1}^\mu = \hat{x}_k^\mu + \frac{\partial \gamma^\mu(\hat{x}_k, \hat{v}_k, 0)}{\partial s} h - \frac{1}{2}\Gamma_{\alpha\beta}^\mu(\hat{x}_k) (\hat{x}_k^\alpha - \hat{x}_{k-1}^\alpha) (\hat{x}_k^\beta - \hat{x}_{k-1}^\beta)\notag\\
		=&  \hat{x}_k^\mu + \frac{\partial \gamma^\mu(\hat{x}_k, \hat{v}_k, 0)}{\partial s} h - \frac{1}{2}\Gamma_{\alpha\beta}^\mu(\hat{x}_k) \bigg[h \frac{\partial \gamma^\alpha(\hat{x}_{k-1}, \hat{v}_{k-1},0)}{\partial s} + \frac{1}{2}h^2 \Gamma_{ab}^\alpha(\hat{x}_{k-1}) f^a(t_{k-1}, \hat{x}_{k-1})\notag\\
		&f^b(t_{k-1}, \hat{x}_{k-1})\bigg]\cdot\left[h \frac{\partial \gamma^\beta(\hat{x}_{k-1}, \hat{v}_{k-1},0)}{\partial s} + \frac{1}{2}h^2 \Gamma_{cd}^\beta(\hat{x}_{k-1}) f^c(t_{k-1}, \hat{x}_{k-1})f^d(t_{k-1}, \hat{x}_{k-1})\right]\notag\\
		=& \hat{x}_k^\mu + \frac{\partial \gamma^\mu(\hat{x}_k, \hat{v}_k, 0)}{\partial s} h - \frac{1}{2}\Gamma_{\alpha\beta}^\mu \frac{\partial \gamma^\alpha(\hat{x}_{k-1}, \hat{v}_{k-1},0)}{\partial s}\frac{\partial \gamma^\beta(\hat{x}_{k-1}, \hat{v}_{k-1},0)}{\partial s}h^2 \notag\\
		& +\Phi^\mu(h, \hat{x}_k, \hat{x}_{k-1}, t_{k-1}, \hat{v}_{k-1})\notag\\
		=& \hat{x}_k^\mu + \frac{\partial \gamma^\mu(\hat{x}_k, \hat{v}_k, 0)}{\partial s} h + \frac{1}{2} \frac{\partial^2 \gamma^\mu(\hat{x}_{k-1}, \hat{v}_{k-1},0)}{\partial s^2}h^2 + \Phi^\mu(h, \hat{x}_k, \hat{x}_{k-1}, t_{k-1}, \hat{v}_{k-1}),\label{eq:thm1_4}
		\end{align}
		where the last line utilizes geodesic equation~\eqref{eq:geo} and $\Phi^\mu(h, \hat{x}_k, \hat{x}_{k-1}, t_{k-1}, \hat{v}_{k-1})$ is
		\begin{align*}
		&h^3 \underbrace{\frac{1}{2} \Gamma^{\mu}_{\alpha\beta}(\hat{x}_k) \Gamma_{ab}^\alpha(\hat{x}_{k-1})f^af^b \frac{\partial \gamma^\beta(\hat{x}_{k-1}, \hat{v}_{k-1},0)}{\partial s}}_{=:\Phi_1(\hat{x}_k, \hat{x}_{k-1}, \hat{v}_{k-1}, t_{k-1})}\\
		+& h^4 \underbrace{\frac{1}{8} \Gamma_{\alpha\beta}^\mu(\hat{x}_k) \Gamma_{ab}^\alpha(\hat{x}_{k-1})\Gamma_{cd}^\beta (\hat{x}_{k-1}) f^a f^b f^c f^d}_{=:\Phi_2(\hat{x}_k, \hat{x}_{k-1}, t_{k-1})},
		\end{align*}
		where $f^i = f^i(t_{k-1}, \hat{x}_{k-1}), i\in\{a,b,c,d\}$. Both $\Phi_1$ and $\Phi_2$ are $C^\infty$ functions on compact sets. As a result, extreme value theorem states that there exist constants $A$ and $B$ so that $\sup \norm{\Phi_1} = A$ and $\sup \norm{\Phi_2} = B$.
		
		Subtracting \eqref{eq:thm1_4} from \eqref{eq:thm1_2} and letting $\theta_k = x_k - \hat{x}_k$ we obtain
		\begin{align*} 		
		\theta_{k+1} = \theta_{k}  &+ \left[\frac{\partial \gamma(x_k, v_k, 0)}{\partial s} - \frac{\partial \gamma(\hat{x}_k, \hat{v}_k, 0)}{\partial s}\right] h \\
		&+ \frac{1}{2}\left[ \frac{\partial^2 \gamma(x_k, v_k, 0)}{\partial s^2} - \frac{\partial^2 \gamma(\hat{x}_k, \hat{v}_k, 0)}{\partial s^2} \right]h^2\\
		&+ h^3 \Phi_1(\hat{x}_k, \hat{x}_{k-1}, \hat{v}_{k-1}, t_{k-1}) + h^4 \Phi_2(\hat{x}_k, \hat{x}_{k-1}, t_{k-1}),
		\end{align*}
		and the error can be bounded by
		\begin{align*}
		\norm{\theta_{k+1}} &\leq \norm{\theta_k} + h (\Gamma_1 \norm{\theta_k} + \Gamma_2 \norm{v_k - \hat{v}_k}) + \frac{1}{2} h^2 (\Gamma_3 \norm{\theta_k} + \Gamma_4 \norm{v_k - \hat{v}_k}) + Ah^3 + Bh^4\\
		&\leq \norm{\theta_k} + h (\Gamma_1 \norm{\theta_k} + L_f\Gamma_2\norm{\theta_k}) + \frac{1}{2}h^2 (\Gamma_3 \norm{\theta_k} + L_f\Gamma_4\norm{\theta_k}) + Ah^3 + Bh^4\\
		&\leq (1 + C_1 h + \frac{1}{2}C_2 h^2)\norm{\theta_k} + Ah^3 + Bh^4\\
		&\leq \cdots\\
		&\leq \left(1 + C_1 h + \frac{1}{2}C_2 h^2\right)^k \left( \norm{\theta_0} + \frac{Ah^2 + Bh^3}{C_1 + \frac{1}{2} C_2 h} \right) - \frac{Ah^2 + Bh^3}{C_1 + \frac{1}{2} C_2 h}\\
		&\leq \left(e^{C_1 kh + \frac{1}{2}C_2 kh^2} - 1\right) \frac{Ah^2 + Bh^3}{C_1 + \frac{1}{2}C_2 h}\\
		&\leq \left( e^{C_1 T + \frac{1}{2} C_2 T^2} - 1 \right)\frac{Ah^2 + Bh^3}{C_1} = \mcal{O}(h^2).
		\end{align*}
		Finally, $\norm{\hat{e}_k} \leq \norm{e_k} + \norm{\theta_k} = \mcal{O}(h^2)$ as $h\rightarrow 0$. This shows that faster geodesic correction converges to the invariant solution of Riemannian Euler method with a 2nd-order rate and the exact solution of ODE in 1st order.
\end{proof}

\section{Derivations of Connections for Different Losses}\label{app:losses}
In this section we show how to derive the formulas of Fisher information matrices and Levi-Civita connections for three common losses used in our experiments.

\subsection{Squared Loss}

The squared loss is induced from negative log-likelihood of the probabilistic model
\begin{align*}
p_\theta(\mbf{t}\mid\mbf{x}) = \prod_{i=1}^o \mcal{N}(t_i\mid y_i, \sigma^2),
\end{align*}
and the log-likelihood is
\begin{align*}
\ln p_\theta(\mbf{t}\mid\mbf{x}) = -\frac{1}{2\sigma^2} \sum_{i=1}^o (t_i - y_i)^2 + \text{const}.
\end{align*}

According to the definition of Fisher information matrix,
\begin{align*}
g_{\mu\nu} &= \mbb{E}_{q(\mbf{x})} [\mbb{E}_{p_\theta(\mbf{t}\mid\mbf{x})}[\partial_\mu \ln p_\theta(\mbf{t}\mid\mbf{x})\partial_\nu \ln p_\theta(\mbf{t}\mid\mbf{x})]]\\
&= \mbb{E}_{q(\mbf{x})}\left[\mbb{E}_{p_\theta(\mbf{t}\mid\mbf{x})}\left[\frac{1}{\sigma^4} \sum_{ij} (t_i-y_i)(t_j-y_j)\partial_\mu y_i \partial_\nu y_j\right]\right]\\
&= \mbb{E}_{q(\mbf{x})}\left[\frac{1}{\sigma^4} \sum_{ij} \delta_{ij} \sigma^2 \partial_\mu y_i \partial_\nu y_j\right]\\
&= \frac{1}{\sigma^2}\sum_{i=1}^o \mbb{E}_{q(\mbf{x})}\left[\partial_\mu y_i \partial_\nu y_j\right].
\end{align*}

To compute the Levi-Civita connection $\Gamma_{\alpha\beta}^\mu$, we first calculate the following component
\begin{align*}
&\mbb{E}_{p_\theta(\mbf{t}\mid\mbf{x})}[\partial_\nu \ln p_\theta(\mbf{t}\mid\mbf{x}) \partial_\alpha\partial_\beta \ln p_\theta(\mbf{t}\mid\mbf{x})]\\
=&\mbb{E}_{p_\theta(\mbf{t}\mid\mbf{x})}\bigg[ \bigg( \frac{1}{\sigma^2}\sum_{i=1}^o (t_i-y_i)\partial_\nu y_i \bigg)\bigg( \frac{1}{\sigma^2} \sum_{j=1}^o -\partial_\alpha y_j \partial_\beta y_j + (t_j-y_j)\partial_\alpha\partial_\beta y_j \bigg)\bigg]\\
=&\frac{1}{\sigma^4}\mbb{E}_{p_\theta(\mbf{t}\mid\mbf{x})}\bigg[ \sum_{ij}(t_i-y_i)(t_j-y_j)\partial_\nu y_i \partial_\alpha\partial_\beta y_j\bigg]\\
=&\frac{1}{\sigma^4}\mbb{E}_{p_\theta(\mbf{t}\mid\mbf{x})}\bigg[ \sum_{ij} \sigma^2 \delta_{ij}\partial_\nu y_i \partial_\alpha\partial_\beta y_j\bigg]\\
=&\frac{1}{\sigma^2}\sum_{i=1}^o\mbb{E}_{p_\theta(\mbf{t}\mid\mbf{x})}\bigg[ \partial_\nu y_i \partial_\alpha\partial_\beta y_i\bigg].
\end{align*}

The second component of $\Gamma_{\alpha\beta}^\mu$ we have to consider is
\begin{align*}
&\mbb{E}_{p_\theta(\mbf{t}\mid\mbf{x})}\left[\frac{1}{2}\partial_\nu \ln p_\theta(\mbf{t}\mid\mbf{x}) \partial_\alpha \ln p_\theta(\mbf{t}\mid\mbf{x}) \partial_\beta \ln p_\theta(\mbf{t}\mid\mbf{x})\right]\\
=&\frac{1}{2\sigma^6}\sum_{ijk}\mbb{E}_{p_\theta(\mbf{t}\mid\mbf{x})}[(t_i-y_i)(t_j-y_j)(t_k-y_k)\partial_\nu y_i \partial_\alpha y_i \partial_\beta y_i]\\
=&0,
\end{align*}
where the last equality uses the property that third moment of a Gaussian distribution is 0.

Combining the above two components, we obtain the Levi-Civita connection
\begin{align*}
\Gamma_{\alpha\beta}^\mu = \frac{1}{\sigma^2} \sum_{i=1}^o g^{\mu\nu} \mbb{E}_{q(\mbf{x})} [\partial_\nu y_i \partial_\alpha\partial_\beta y_i].
\end{align*}

\subsection{Binary Cross-Entropy}

The binary cross-entropy loss is induced from negative log-likelihood of the probabilistic model
\begin{align*}
p_\theta(\mbf{t}\mid\mbf{x}) = \prod_{i=1}^o y_i^{t_i} (1- y_i)^{1-t_i},
\end{align*}
and the log-likelihood is
\begin{align*}
\ln p_\theta(\mbf{t}\mid\mbf{x}) = \sum_{i=1}^o t_i \ln y_i + (1-t_i) \ln(1-y_i).
\end{align*}

According to the definition of Fisher information matrix,
\begin{align*}
g_{\mu\nu} &= \mbb{E}_{q(\mbf{x})} [\mbb{E}_{p_\theta(\mbf{t}\mid\mbf{x})}[\partial_\mu \ln p_\theta(\mbf{t}\mid\mbf{x})\partial_\nu \ln p_\theta(\mbf{t}\mid\mbf{x})]]\\
&= \mbb{E}_{q(\mbf{x})}\left[\mbb{E}_{p_\theta(\mbf{t}\mid\mbf{x})}\left[\sum_{ij} \frac{(t_i-y_i)(t_j-y_j)}{y_iy_j(1-y_i)(1-y_j)}\partial_\mu y_i \partial_\nu y_j\right]\right]\\
&= \mbb{E}_{q(\mbf{x})}\left[\sum_{ij}\delta_{ij}\frac{y_i(1-y_i)}{y_iy_j(1-y_i)(1-y_j)}\partial_\mu y_i \partial_\nu y_j\right]\\
&= \sum_{i=1}^o \mbb{E}_{q(\mbf{x})}\left[\frac{1}{y_i(1-y_i)}\partial_\mu y_i \partial_\nu y_j\right].
\end{align*}

To compute the Levi-Civita connection $\Gamma_{\alpha\beta}^\mu$, we first calculate the following component
\begin{align*}
&\mbb{E}_{p_\theta(\mbf{t}\mid\mbf{x})}[\partial_\nu \ln p_\theta(\mbf{t}\mid\mbf{x}) \partial_\alpha\partial_\beta \ln p_\theta(\mbf{t}\mid\mbf{x})]\\
=&\mbb{E}_{p_\theta(\mbf{t}\mid\mbf{x})}\bigg[ \bigg( \sum_{i=1}^o \frac{t_i-y_i}{y_i(1-y_i)}\partial_\mu y_i\bigg)\bigg(\sum_{j=1}^o -\frac{(t_j-y_j)^2}{y_j^2(1-y_j)^2}\partial_\alpha y_j\partial_\beta y_j + \frac{t_j-y_j}{y_j(1-y_j)}\partial_\alpha \partial_\beta y_j\bigg)\bigg]\\
=&\mbb{E}_{p_\theta(\mbf{t}\mid\mbf{x})}\left[ - \sum_{ij} \frac{(t_i-y_i)(t_j-y_j)^2}{y_i(1-y_i)y_j^2(1-y_j)^2}\partial_\mu y_i \partial_\alpha y_j \partial_\beta y_j  + \sum_{ij}\frac{(t_i-y_i)(t_j-y_j)}{y_i(1-y_i)y_j(1-y_j)}\partial_\mu y_i \partial_\alpha \partial_\beta y_j\right]\\
=&\mbb{E}_{p_\theta(\mbf{t}\mid\mbf{x})}\left[ - \sum_{i=1}^o \frac{(t_i-y_i)^3}{y_i^3(1-y_i)^3}\partial_\mu y_i \partial_\alpha y_i \partial_\beta y_i  + \sum_{i=1}^o\frac{(t_i-y_i)^2}{y_i^2(1-y_i)^2}\partial_\mu y_i \partial_\alpha \partial_\beta y_i\right]\\
=& - \sum_{i=1}^o \frac{(1-y_i)^3y_i + (-y_i)^3(1-y_i)}{y_i^3(1-y_i)^3}\partial_\mu y_i \partial_\alpha y_i \partial_\beta y_i  + \sum_{i=1}^o\frac{1}{y_i(1-y_i)}\partial_\mu y_i \partial_\alpha \partial_\beta y_i\\
=&\sum_{i=1}^o \frac{2y_i -1}{y_i^2(1-y_i)^2}\partial_\mu y_i \partial_\alpha y_i \partial_\beta y_i  + \sum_{i=1}^o\frac{1}{y_i(1-y_i)}\partial_\mu y_i \partial_\alpha \partial_\beta y_i.
\end{align*}
where in the first line we use the equality $(t_j-y_j)^2=(t_j-2t_j y_j+y_j^2)$ which holds given that $t_j \in \{0,1\}$.

The second component of $\Gamma_{\alpha\beta}^\mu$ is
\begin{align*}
&\mbb{E}_{p_\theta(\mbf{t}\mid\mbf{x})}\left[\frac{1}{2}\partial_\nu \ln p_\theta(\mbf{t}\mid\mbf{x}) \partial_\alpha \ln p_\theta(\mbf{t}\mid\mbf{x}) \partial_\beta \ln p_\theta(\mbf{t}\mid\mbf{x})\right]\\
=& \frac{1}{2}\sum_{ijk}\mbb{E}_{p_\theta(\mbf{t}\mid\mbf{x})}\left[ \frac{(t_i-y_i)(t_j-y_j)(t_k-y_k)}{y_iy_jy_k(1-y_i)(1-y_j)(1-y_k)} \partial_\nu y_i \partial_\alpha y_j \partial_\beta y_k \right]\\
=&\frac{1}{2} \sum_{i=1}^o \mbb{E}_{p_\theta(\mbf{t}\mid\mbf{x})}\left[ \frac{(t_i-y_i)^3}{y_i^3(1-y_i)^3} \partial_\nu y_i \partial_\alpha y_i \partial_\beta y_i \right]\\
=&\sum_{i=1}^o  \frac{1-2y_i}{2y_i^2(1-y_i)^2} \partial_\nu y_i \partial_\alpha y_i \partial_\beta y_i,
\end{align*}

Combining the above two components, we obtain the Levi-Civita connection
\begin{align*}
\Gamma^\mu_{\alpha\beta} = g^{\mu\nu} \sum_{i=1}^o \mbb{E}_{q(\mbf{x})}\bigg[ \frac{2y_i-1}{2y_i^2(1-y_i)^2} \partial_\nu y_i \partial_\alpha y_i \partial_\beta y_i + \frac{1}{y_i(1-y_i)}\partial_\nu y_i\partial_\alpha\partial_\beta y_i \bigg].
\end{align*}

\subsection{Multi-Class Cross-Entropy}

The multi-class cross-entropy loss is induced from negative log-likelihood of the probabilistic model
\begin{align*}
p_\theta(\mbf{t}\mid\mbf{x}) = \prod_{i=1}^o y_i^{t_i},
\end{align*}
and the log-likelihood is
\begin{align*}
\ln p_\theta(\mbf{t}\mid\mbf{x}) = \sum_{i=1}^{o} t_i \ln y_i.
\end{align*}

According to the definition of Fisher information matrix,
\begin{align*}
g_{\mu\nu} &= \mbb{E}_{q(\mbf{x})} [\mbb{E}_{p_\theta(\mbf{t}\mid\mbf{x})}[\partial_\mu \ln p_\theta(\mbf{t}\mid\mbf{x})\partial_\nu \ln p_\theta(\mbf{t}\mid\mbf{x})]]\\
&= \mbb{E}_{q(\mbf{x})}\left[\mbb{E}_{p_\theta(\mbf{t}\mid\mbf{x})}\left[\sum_{ij}\frac{t_i t_j}{y_i y_j}\partial_\mu y_i \partial_\nu y_j\right]\right]\\
&= \mbb{E}_{q(\mbf{x})}\left[ \sum_{i=1}^o \mbb{E}_{p_\theta(\mbf{t}\mid\mbf{x})}\left[\frac{t_i^2}{y_i^2} \right]\partial_\mu y_i \partial_\nu y_i \right]\\
&=\sum_{i=1}^o \mbb{E}_{q(\mbf{x})}\left[ \frac{1}{y_i} \partial_\mu y_i \partial_\nu y_i \right].
\end{align*}

To compute the Levi-Civita connection $\Gamma_{\alpha\beta}^\mu$, we first calculate the following component
\begin{align*}
&\mbb{E}_{p_\theta(\mbf{t}\mid\mbf{x})}[\partial_\nu \ln p_\theta(\mbf{t}\mid\mbf{x}) \partial_\alpha\partial_\beta \ln p_\theta(\mbf{t}\mid\mbf{x})]\\
=&\mbb{E}_{p_\theta(\mbf{t}\mid\mbf{x})}\bigg[ \bigg( \sum_{i=1}^o \frac{t_i}{y_i} \partial_\nu y_i \bigg)\bigg( \sum_{j=1}^o -\frac{t_j}{y_j^2} \partial_\alpha y_j \partial_\beta y_j + \frac{t_j}{y_j} \partial_\alpha \partial_\beta y_j \bigg)\bigg]\\
=&\mbb{E}_{p_\theta(\mbf{t}\mid\mbf{x})}\bigg[ \sum_{i=1}^o -\frac{t_i^2}{y_i^3}\partial_\nu y_i \partial_\alpha y_i \partial_\beta y_i + \frac{t_i^2}{y_i^2} \partial_\nu y_i \partial_\alpha \partial_\beta y_i \bigg]\\
=&\sum_{i=1}^o -\frac{1}{y_i^2} \partial_\nu y_i \partial_\alpha y_i \partial_\beta y_i + \frac{1}{y_i} \partial_\nu y_i \partial_\alpha \partial_\beta y_i.
\end{align*}

The second component of $\Gamma_{\alpha\beta}^\mu$ we have to consider is
\begin{align*}
&\mbb{E}_{p_\theta(\mbf{t}\mid\mbf{x})}\left[\frac{1}{2}\partial_\nu \ln p_\theta(\mbf{t}\mid\mbf{x}) \partial_\alpha \ln p_\theta(\mbf{t}\mid\mbf{x}) \partial_\beta \ln p_\theta(\mbf{t}\mid\mbf{x})\right]\\
=& \mbb{E}_{p_\theta(\mbf{t}\mid\mbf{x})}\left[\sum_{ijk}\frac{t_it_jt_k}{2y_iy_jy_k}\partial_\nu y_i \partial_\alpha y_j \partial_\beta y_k\right]\\
=& \mbb{E}_{p_\theta(\mbf{t}\mid\mbf{x})}\left[\sum_{i=1}^o\frac{t_i^3}{2y_i^3}\partial_\nu y_i \partial_\alpha y_i \partial_\beta y_i\right]\\
=& \sum_{i=1}^o \frac{1}{2y_i^2} \partial_\nu y_i \partial_\alpha y_i \partial_\beta y_i.
\end{align*}

Combining the above two components, we obtain the Levi-Civita connection
\begin{align*}
\Gamma^\mu_{\alpha\beta} = g^{\mu\nu} \sum_{i=1}^o \mbb{E}_{q(\mbf{x})}\bigg[ \frac{1}{y_i} \partial_\nu y_i \partial_\alpha \partial_\beta y_i - \frac{1}{2y_i^2}\partial_\nu y_i \partial_\alpha y_i \partial_\beta y_i \bigg].
\end{align*}

\section{Computing Connection Products via Backpropagation}\label{app:computation}
It is not tractable to compute $\Gamma^\mu_{\alpha\beta}$ for large neural networks. Fortunately, to evaluate~(\ref{main:geo:eq}) we only need to know $\Gamma^{\mu}_{\alpha\beta} \dot{\gamma}^\alpha \dot{\gamma}^\beta$. For the typical losses in Proposition \ref{prop:l2}, this expression contains two main terms:
\begin{enumerate}
	\item $\sum_{i=1}^o \lambda_i \partial_\nu y_i \partial_\alpha\partial_\beta y_i \dot\gamma^\alpha \dot\gamma^\beta$. Note that $\partial_\alpha \partial_\beta y_i \dot\gamma^\alpha \dot\gamma^\beta$ is the directional second derivative of $y_i$ along the direction of $\dot\gamma$ (it's a scalar).
	\item $\sum_{i=1}^o \lambda_i \partial_\nu y_i \partial_\alpha y_i \partial_\beta y_i \dot\gamma^\alpha \dot\gamma^\beta$. Note that $\partial_\alpha y_i \partial_\beta y_i \dot\gamma^\alpha \dot\gamma^\beta = (\partial_\alpha y_i \dot\gamma^\alpha)^2$ (recall Einstein's notation), where $\partial_\alpha y_i \dot\gamma^\alpha$ is the directional derivative of $y_i$ along the direction of $\dot\gamma$.
\end{enumerate}

After obtaining the directional derivatives of $y_i$ (scalars $\mu_i$), both terms have the form of $\sum_{i=1}^o \lambda_i \mu_i \partial_\nu y_i$. It can be computed via backpropagation with loss function $L = \sum_{i=1}^o \lambda_i \mu_i y_i$ while treating $\lambda_i \mu_i$ as constants.

Inspired by the ``Pearlmutter trick'' for computing Hessian-vector and curvature matrix-vector products~\citep{pearlmutter1994fast,schraudolph2002fast}, we propose a similar method to compute directional derivatives and connections. 

As a first step, we use the following notations. Given an input $\mbf{x}$ and parameters $\theta = (W_1,\cdots,W_l,b_1,\cdots,b_l)$, a feed-forward neural network computes its output $\mbf{y}(\mbf{x},\theta) = a_l$ by the recurrence
\begin{align}
s_i &= W_i a_{i-1} + b_i \label{eqn:nn1}\\
a_i &= \phi_i(s_i)\label{eqn:nn2},
\end{align}
where $W_i$ is the weight matrix, $b_i$ is the bias, and $\phi_i(\cdot)$ is the activation function. Here $a_i$, $b_i$ and $s_i$ are all vectors of appropriate dimensions. The loss function $L(\mbf{t}, \mbf{y})$ measures the distance between the ground-truth label $\mbf{t}$ of $\mbf{x}$ and the network output $\mbf{y}$. For convenience, we also define 
\begin{align*}
\mcal{D}(v) &= \frac{\ud L(\mbf{t},\mbf{y})}{\ud v}\\
\mcal{R}_v (g(\theta)) &= \lim_{\epsilon\rightarrow 0} \frac{1}{\epsilon} [g(\theta+\epsilon v) - g(\theta)]\\
\mcal{S}_v (g(\theta)) &= \lim_{\epsilon\rightarrow 0} \frac{1}{\epsilon^2}[g(\theta+2\epsilon v) - 2g(\theta+\epsilon v) + g(\theta)]\\
&= \mcal{R}_v(\mcal{R}_v(g(\theta))),
\end{align*}
which represents the gradient of $L(\mbf{t},\mbf{y})$, directional derivative of $g(\theta)$ and directional second derivative of $g(\theta)$ along the direction of $v$ respectively.

The following observation is crucial for our calculation:
\begin{proposition}\label{prop:property}
	For any differentiable scalar function $g(\theta), g_1(\theta), g_2(\theta), f(x)$ and vector $v$, we have
	\begin{align*}
	\mcal{R}_v(g_1 + g_2) &= \mcal{R}_v(g_1) + \mcal{R}_v(g_2)\\
	\mcal{R}_v(g_1 g_2) &= \mcal{R}_v (g_1)g_2 + g_1\mcal{R}_v(g_2)\\
	\mcal{R}_v(f(g)) &= f'\mcal{R}_v(g)\\
	\mcal{S}_v(g_1 + g_2) &= \mcal{S}_v(g_1) + \mcal{S}_v(g_2)\\
	\mcal{S}_v(g_1 g_2) &= \mcal{S}_v(g_1)g_2 + 2\mcal{R}_v(g_1)\mcal{R}_v(g_2) + g_1\mcal{S}_v(g_2)\\
	\mcal{S}_v(f(g)) &= f''\mcal{R}_v(g)^2 + f'\mcal{S}_v(g)
	\end{align*}
\end{proposition} 

Using those new notations, the directional derivatives $\partial_\alpha \partial_\beta y_i \dot\gamma^\alpha \dot\gamma^\beta$, $\partial_\alpha y_i \dot\gamma^\alpha$ can be written as $\mcal{S}_{\dot{\gamma}}(y_i)$ and $\mcal{R}_{\dot{\gamma}}(y_i)$. We can obtain recurrent equations for them by applying \propref{prop:property} to \eqref{eqn:nn1} and $\eqref{eqn:nn2}$. The results are
\begin{align*}
\mcal{R}_{\dot{\gamma}}(s_i) &= \mcal{R}_{\dot{\gamma}}(W_i)a_{i-1} + W_i\mcal{R}_{\dot{\gamma}}(a_{i-1}) + \mcal{R}_{\dot{\gamma}}(b_i)\\
\mcal{R}_{\dot{\gamma}}(a_i) &= \phi'(s_i)\odot \mcal{R}_{\dot{\gamma}}(s_i)\\
\mcal{S}_{\dot{\gamma}}(s_i) &= \mcal{S}_{\dot{\gamma}}(W_i)a_{i-1} + 2\mcal{R}_{\dot{\gamma}}(W_i)\mcal{R}_{\dot{\gamma}}(a_{i-1}) \\
&\quad + W_i\mcal{S}_{\dot{\gamma}}(a_{i-1}) + \mcal{S}_{\dot{\gamma}}(b_i)\\
\mcal{S}_{\dot{\gamma}}(a_i) &= \phi''(s_i)\odot \mcal{R}_{\dot{\gamma}}(s_i)^2 + \phi'(s_i)\odot \mcal{S}_{\dot{\gamma}}(s_i),
\end{align*}
which can all be computed during the forward pass, given that $\mcal{R}_{\dot{\gamma}}(W_i) = \dot{\gamma}$ and $\mcal{S}_{\dot{\gamma}}(W_i) = 0$.
Based on the above recurrent rules, we summarize our algorithms for those two terms of connections in Alg.~\ref{alg:term1} and Alg.~\ref{alg:term2}.

\begin{algorithm}
	\caption{Calculating $\sum_{i=1}^o \lambda_i\partial_\nu y_i \partial_{\alpha}\partial_\beta y_i \dot\gamma^\alpha \dot\gamma^\beta$ (term 1)}\label{alg:term1}
	\begin{algorithmic}[1]		
		\Require{$\dot\gamma$. Here we abbreviate $\mcal{R}_{\dot\gamma}$ to $\mcal{R}$ and $\mcal{S}_{\dot\gamma}$ to $\mcal{S}$.}
		\item[]
		\State{$a_0 \gets x$}
		\State{$\mcal{R}(a_0) \gets 0$}
		\State{$\mcal{S}(a_0) \gets 0$}
		\item[]
		\For{$i \gets 1$ to $l$}
		\Comment{forward pass}
		\State{$s_i \gets W_i a_{i-1} + b_i$}
		\State{$a_i \gets \phi_i(s_i)$}
		\State{$\mcal{R}(s_i) \gets \mcal{R}(W_i)a_{i-1} + W_i\mcal{R}(a_{i-1}) + \mcal{R}(b_i)$}
		\State{$\mcal{R}(a_i) \gets \phi'(s_i) \odot \mcal{R}(s_i)$}
		\State{$\mcal{S}(s_i) \gets \mcal{S}(W_i)a_{i-1} + 2\mcal{R}(W_i)\mcal{R}(a_{i-1}) + W_i\mcal{S}(a_{i-1}) + \mcal{S}(b_i)$}
		\State{$\mcal{S}(a_i) \gets \phi''(s_i)\odot \mcal{R}(s_i)^2 + \phi'(s_i)\odot \mcal{S}(s_i)$}
		\EndFor
		\item[]
		\State{Compute $\vec{\lambda}$ from $a_l$}
		\State{$\mcal{D}(a_l) = \vec{\lambda} \mcal{S}(a_l)$}
		\item[]
		\For{$i = l$ to 1}
		\Comment{backward pass}
		\State{$\mcal{D}(s_i) \gets \mcal{D}(a_i) \odot \phi'(s_i)$}
		\State{$\mcal{D}(W_i) \gets \mcal{D}(s_i) a_{i-1}^\intercal$}
		\State{$\mcal{D}(b_i) \gets \mcal{D}(s_i)$}
		\State{$\mcal{D}(a_{i-1}) \gets W_i^\intercal \mcal{D}(s_i)$}
		\EndFor
		\item[]
		\Return{$ (\mcal{D}(W_1),\cdots,\mcal{D}(W_l),\mcal{D}(b_1),\cdots,\mcal{D}(b_l))$.}
	\end{algorithmic}
\end{algorithm}

\begin{algorithm}
	\caption{Calculating $\sum_{i=1}^o \lambda_i \partial_\nu y_i \partial_{\alpha}y_i \partial_\beta y_i \dot\gamma^\alpha \dot\gamma^\beta$ (term 2)}\label{alg:term2}
	\begin{algorithmic}[1]
		\Require{$\dot\gamma$. Here we abbreviate $\mcal{R}_{\dot\gamma}$ to $\mcal{R}$.}
		\item[]
		\State{$a_0 = x$}
		\State{$\mcal{R}(a_0) = 0$}
		\item[]
		\For{$i \gets 1$ to $l$}
		\Comment{forward pass}
		\State{$s_i \gets W_i a_{i-1} + b_i$}
		\State{$a_i \gets \phi(s_i)$}
		\State{$\mcal{R}(s_i) \gets \mcal{R}(W_i)a_{i-1} + W_i\mcal{R}(a_{i-1}) + \mcal{R}(b_i)$}
		\State{$\mcal{R}(a_i) \gets \phi'(s_i) \odot \mcal{R}(s_i)$}
		\EndFor
		\item[]
		\State{Compute $\vec{\lambda}$ from $a_l$}
		\State{$\mcal{D}(a_l) = \vec{\lambda} \mcal{R}(a_l)^2$}
		\item[]
		\For{$i = l$ to 1}
		\Comment{backward pass}
		\State{$\mcal{D}(s_i) \gets \mcal{D}(a_i) \odot \phi'(s_i)$}
		\State{$\mcal{D}(W_i) \gets \mcal{D}(s_i) a_{i-1}^\intercal$}
		\State{$\mcal{D}(b_i) \gets \mcal{D}(s_i)$}
		\State{$\mcal{D}(a_{i-1}) \gets W_i^\intercal \mcal{D}(s_i)$}
		\EndFor
		\item[]
		\Return{$(\mcal{D}(W_1),\cdots,\mcal{D}(W_l),\mcal{D}(b_1),\cdots,\mcal{D}(b_l))$.}
	\end{algorithmic}
\end{algorithm}

\section{Practical Considerations}\label{app:settings}
Practically, the Fisher information matrix could be ill-conditioned for inversion. In experiments, we compute $[g_{\mu\nu} + \epsilon \operatorname{diag}(g_{\mu\nu})]^{-1}\partial_\nu L$ instead of $(g_{\mu\nu})^{-1} \partial_\nu L$, where $\epsilon$ is the damping coefficient and $\operatorname{diag}(g_{\mu\nu})$ is the diagonal part of $g_{\mu\nu}$. When $g_{\mu\nu}$ is too large to be inverted accurately, we use truncated conjugate gradient for solving the corresponding linear system.

Moreover, in line with the pioneering work of~\citet{martens2010deep}, we use backtracking search to adaptively shrink the step size and adopt a Levenberg-Marquardt style heuristic for adaptively choosing the damping coefficient. 

As pointed out in~\citet{ollivier2013riemannian}, there are also two other sources of invariance loss, initialization and damping. Simple random initialization obviously depends on the network architecture. Unfortunately, there is no clear way to make it independent of parameterization. Large damping wipes out small eigenvalue directions and swerves optimization towards na\"{i}ve gradient descent, which is not invariant. When the damping coefficient selected according to the Marquardt heuristic~\citep{marquardt1963algorithm} is very large, it becomes meaningless to use either midpoint integrator or geodesic correction. In the experiments of training deep neural networks, we found it beneficial to set a threshold for the damping coefficient and switch off midpoint integrator or geodesic correction at the early stage of optimization when damping is very large.

\section{Additional Details on Experimental Evaluations}

\subsection{Settings for Deep Neural Network Training}\label{app:exp1}
For the deep network experiments, we use the hyper-parameters in \citet{martens2010deep} as a reference, the modifications are that we fix the maximum number of CG iterations to 50 and maximum number of epochs to 140 for all algorithms and datasets. The initial damping coefficient is 45 across all tasks, and the damping thresholds for CURVES, MNIST and FACES are set to 5, 10 and 0.1 respectively. As mentioned in Appendix~\ref{app:settings}, we use a threshold on the damping to switch on / off our corrections. However, in reality our methods will only be switched off for a small number of iterations in the early stage of training. Note that more careful tuning of these thresholds, \eg, using different thresholds for different acceleration methods, may lead to better results. 

Since both midpoint integrator and geodesic correction are direct modifications of natural gradient method, we incorporate all the improvements in~\citet{martens2010deep}, including sparse initialization, CG iteration backtracking, \etc. 
For determining the learning rate $h\lambda$, we use the default value $h\lambda = 1$ with standard backtracking search. 
To highlight the effectiveness of the algorithmic improvements we introduced, the same set of hyper-parameters and random seed is used across all algorithms on all datasets. 

For deep autoencoders, network structures are the same as in~\citet{hinton2006reducing} and \citet{martens2010deep} and we adopt their training / test partitions and choice of loss functions. For deep classifiers on MNIST, the network structure is 784-1000-500-250-30-10, all with fully connected layers, and as preprocessing, we center and normalize all training and test data.

Deep autoencoders for CURVES and MNIST datasets are trained with binary cross-entropy losses while FACES is trained with squared loss. All results are reported in squared losses. Although there is discrepancy between training and test losses, they align with each other pretty well and thus we followed the setting in ~\cite{martens2010deep}. According to our observation, performance is robust to different random seeds and the learning curves measured by errors on training and test datasets are similar, except that we slightly overfitted FACES dataset.

Our implementation is based on the MATLAB code provided by~\cite{martens2010deep}. However, we used MATLAB Parallel Computing Toolbox for GPU, instead of the Jacket package used in \cite{martens2010deep}, because Jacket is not available anymore. Computation times are not directly comparable as the Parallel Computing Toolbox is considerably slower than Jacket. The programs were run on Titan Xp GPUs.

\subsection{Settings for Model-Free Reinforcement Learning}\label{app:exp2}
We consider common hyperparameter choices for ACKTR as well as our midpoint integrator and geodesic correction methods, where both the policy network and the value network is represented as a two layer fully-connected neural network with 64 neurons in each layer. Specifically, we consider our methods (and subsequent changes to the hyperparameters) only on the policy networks.
We select constant learning rates for each environment since it eliminates the effect of the learning rate schedule in~\citep{wu2017scalable} over sample efficiency. The learning rates are set so that ACKTR achieves the highest episodic reward at 1 million timesteps. We select learning rates of 1.0, 0.03, 0.03, 0.03, 0.3, 0.01 for \texttt{HalfCheetah}, \texttt{Hopper}, \texttt{Reacher}, \texttt{Walker2d}, \texttt{InvertedPendulum}, \texttt{InvertedDoublePendulum} respectively. 
We set momentum to be zero for all methods, since we empirically find that this improves sample efficiency for ACKTR with the fixed learning rate schedule. For example, our ACKTR results for the \texttt{Walker2d} environment is over 1500 for 1 million timesteps, whereas \citep{wu2017scalable} reports no more than 800 for the same number of timesteps (even with the learning rate schedule).

The code is based on OpenAI baselines~\citep{baselines} and connection-vector products are computed with TensorFlow~\citep{abadi2016tensorflow} automatic differentiation.

\section{Experiments on the Small-Curvature Approximation}\label{app:small}

Our geodesic correction is inspired by geodesic acceleration~\citep{transtrum2011geometry}, a method to accelerate the Gauss-Newton algorithm for nonlinear least squares problems. In \cite{transtrum2012geodesic}, geodesic acceleration is derived from a high-order approximation to Hessian under the so-called \emph{small-curvature assumption}. In this section, we demonstrate empirically that the small-curvature approximation generally does not hold for deep neural networks. To this end, we need to generalize the method in~\cite{transtrum2012geodesic} (which is only applicable to square loss) to general losses. 

\subsection{Derivation Based on Perturbation}
It can be shown that Fisher information matrix is equivalent to the Gauss-Newton matrix when the loss function is appropriately chosen. Let's analyze the acceleration terms from this perspective.

Let the loss function be $\mcal{L}(y, f)$ and $z^i (x;\theta)$, $i=1, \cdots, o$ be the top layer values of the neural network. To show the equivalence of Gauss-Newton matrix and Fisher information matrix, we usually require $\mcal{L}$ to also include the final layer activation (non-linearity) applied on $z$~\citep{pascanu2013revisiting,martens2014new}. 
Hence different from $y$, $z$ is usually the value \emph{before} final layer activation function. To obtain the conventional Gauss-Newton update, we analyze the following problem: 

\begin{gather*}
\min_{\delta\theta} \sum_{(x,y) \in S}\mcal{L}(y, z + \partial_j z \delta \theta^j) + \kappa F_{ij} \delta\theta^i \delta\theta^j,
\end{gather*}
where $S$ is the training dataset and $F$ is a metric measuring the distance between two models with parameter difference $\delta\theta$. Note that without loss of generality, we omit $\sum_{(x,y) \in S}$ in the sequel.

By approximating $\mcal{L}(y,\cdot)$ with a second-order Taylor expansion, we obtain
\begin{align*}
\mcal{L}(y,z) + \partial_k \mcal{L}(y,z) \partial_j z^k \delta\theta^j + \frac{1}{2}\partial_m\partial_n \mcal{L}(y,z) \partial_i z^m \partial_j z^n \delta\theta^i \delta\theta^j + \kappa F_{ij} \delta\theta^i \delta\theta^j.
\end{align*}
The normal equations obtained by setting derivatives to $0$ are
\begin{align*}
\left( \partial_m \partial_n \mcal{L} \partial_i z^m \partial_j z^n + 2 \kappa F_{ij} \right) \delta \theta^j = -\partial_k \mcal{L}(y,z) \partial_i z^k,
\end{align*}
which exactly gives the natural gradient update
\begin{align*}
\delta \theta^j_1 = - \left( \partial_m \partial_n \mcal{L} \partial_i z^m \partial_j z^n + \lambda F_{ij} \right)^{-1} \partial_k \mcal{L}(y,z) \partial_i z^k.
\end{align*}
where we fold $2\kappa$ to $\lambda$.
Hence natural gradient is an approximation to the Hessian with linearized model output $z + \partial_j z \delta \theta^j$.

Now let us correct the error of linearizing $z$ using higher order terms, \ie,
\begin{align*}
\mcal{L}\left( y, z + \partial_j z \delta \theta^j + \frac{1}{2} \partial_j \partial_l z \delta \theta^j \delta \theta^l \right) + \lambda F_{ij} \delta\theta^i \delta\theta^j.
\end{align*}
Expanding $\mcal{L}(y,\cdot)$ to second-order gives us
\begin{multline*}
\mcal{L}(y,z) + \partial_k \mcal{L}(y,z)\left(\partial_j z^k \delta\theta^j + \frac{1}{2} \partial_j \partial_l z^k \delta\theta^j \delta\theta^l\right) +\\ \frac{1}{2} \partial_m \partial_n \mcal{L}(y,z) \left(\partial_j z^m \delta\theta^j + \frac{1}{2} \partial_j \partial_l z^m \delta\theta^j \delta\theta^l\right)\left(\partial_k z^n \delta \theta^k + \frac{1}{2} \partial_k \partial_p z^n \delta \theta^k \delta \theta^p\right) + \lambda F_{ij} \delta\theta^i \delta \theta^j.
\end{multline*}

The normal equations are
\begin{multline*}
\partial_k \mcal{L}(y,z) \partial_\mu z^k + (\partial_k \mcal{L}(y,z) \partial_\mu\partial_j z^k + \partial_m \partial_n \mcal{L}(y,z) \partial_\mu z^m \partial_j z^n + \lambda F_{\mu j} ) \delta\theta^j + \\ \left( \partial_m \partial_n \mcal{L}(y,z) \partial_\mu \partial_j z^m \partial_k z^n + \frac{1}{2} \partial_m \partial_n \mcal{L}(y,z) \partial_j \partial_k z^m \partial_\mu z^n \right)\delta\theta^j \delta\theta^k = 0.
\end{multline*}

Let $\delta \theta = \delta \theta_1 + \delta \theta_2$ and assume $\delta\theta_2$ to be small. Dropping in $\delta\theta_1$ will turn the normal equation to
\begin{multline*}
(\partial_k \mcal{L}(y,z) \partial_i \partial_j z^k + \partial_m \partial_n \mcal{L}(y,z) \partial_i z^m \partial_j z^n + \lambda F_{ij}) \delta \theta_2^j + \partial_k \mcal{L}(y,z) \partial_i \partial_j z^k \delta\theta_1^j + \\\left( \partial_m \partial_n \mcal{L}(y,z) \partial_\mu \partial_j z^m \partial_k z^n + \frac{1}{2} \partial_m \partial_n \mcal{L}(y,z) \partial_j \partial_k z^m \partial_\mu z^n \right)\delta \theta_1^j \delta \theta_1^k = 0. 
\end{multline*}

The approximation for generalized Gauss-Newton matrix is $\partial_k \mcal{L}(y,z) \partial_i\partial_j z^k = 0$. After applying it to $\delta \theta_2$, we have
\begin{multline}
	\delta \theta_2^\mu = -(\partial_m \partial_n \mcal{L}(y,z) \partial_i z^m \partial_j z^n + \lambda F_{ij})^{-1} \\\bigg[\left( \partial_m \partial_n \mcal{L}(y,z) \partial_\mu \partial_j z^m \partial_k z^n + \frac{1}{2} \partial_m \partial_n \mcal{L}(y,z) \partial_j \partial_k z^m \partial_\mu z^n \right)\delta \theta_1^j \delta \theta_1^k + \partial_k \mcal{L}(y,z) \partial_\mu \partial_j z^k \delta \theta_1^j\bigg].\label{eqn:perturb}
\end{multline}

If we combine $\partial_k \mcal{L}(y,z) \partial_\mu \partial_j z^k \delta \theta_1^j$ and $[\partial_m \partial_n \mcal{L}(y,z) \partial_\mu \partial_j z^m \partial_k z^n]\delta \theta_1^j \delta \theta_1^k$ and use the following \emph{small-curvature} approximation~\citep{transtrum2012geodesic}
\begin{equation*}
	\partial_l \mcal{L} [\delta_{ml} -  (\partial_\alpha \partial_\beta \mcal{L} \partial_i f^\alpha \partial_k f^\beta)^{-1} \partial_i f^l \partial_m \partial_n \mcal{L} \partial_k z^n]\partial_\mu \partial_j z^m = 0,
\end{equation*}
we will obtain an expression of $\delta \theta_2$ which is closely related to the geodesic correction term.
\begin{equation}
	\delta \theta_2^\mu =  -\frac{1}{2} (\partial_m \partial_n \mcal{L}(y,z) \partial_i z^m \partial_j z^n + \lambda F_{ij})^{-1} \left( \partial_m \partial_n \mcal{L}(y,z) \partial_j \partial_k z^m \partial_\mu z^n \right)\delta \theta_1^j \delta \theta_1^k.\label{eqn:smallcurve}
\end{equation}

The final update rule is
\begin{align*}
\delta\theta = \delta \theta_1 + \delta \theta_2.
\end{align*}

Whenever the loss function satisfies $\mcal{L}(y,z) = -\log r(y\mid z) = - z^\intercal T(y) + \log Z(z)$ we have $\mbf{F}_\theta = \sum_{(x,y) \in S} \partial_m \partial_n \mcal{L}(y,z) \partial_i \partial_j z^m z^n$, \ie, Gauss-Newton method coincides with natural gradient and the Fisher information matrix is the Gauss-Newton matrix. 

Here are the formulas for squared loss and binary cross-entropy loss, where we follow the notation in the main text and denote $y$ as the final network output \emph{after} activation.
\begin{proposition}
For linear activation function and squared loss, we have the following formulas
\begin{align*}
\partial_m \partial_n \mcal{L}(t,f) \partial_\mu \partial_j z^m \partial_k z^n \delta \theta_1^j \delta \theta_1^k &= \frac{1}{\sigma^2} \partial_\mu \partial_j y^m \partial_k y^m \delta \theta_1^j \delta \theta_1^k\\
\frac{1}{2} \partial_m \partial_n \mcal{L}(t,f) \partial_j \partial_k z^m \partial_\mu z^n \delta\theta_1^j \delta\theta_1^k &= \frac{1}{2\sigma^2} \partial_j \partial_k y^m \partial_\mu y^m \delta \theta_1^j \delta \theta_1^k\\
\partial_k \mcal{L}(t,f) \partial_\mu \partial_j z^k \delta \theta_1^j &= \frac{1}{\sigma^2} (y_k - t_k) \partial_\mu \partial_j y^k \delta\theta_1^j.
\end{align*}
In this case, \eqref{eqn:smallcurve} is equivalent to $-\frac{1}{2} \Gamma_{jk}^\mu \delta\theta_1^j \delta\theta_1^k$, which is our geodesic correction term for squared loss.	
\end{proposition}

\begin{proposition}
For sigmoid activation function and binary cross-entropy loss, we have the following terms
\begin{gather*}		
y_i := \operatorname{sigmoid}(z) := \frac{1}{1 + e^{-z^i}}\\
\partial_m \partial_n \mcal{L}(t,f) = \delta_{mn} y_m (1-y_m) \\
\partial_m \partial_n \mcal{L}(t,f) \partial_\mu \partial_j z^m \partial_k z^n \delta \theta_1^j \delta \theta_1^k = \bigg[ \frac{1}{y_m(1-y_m)} \partial_\mu\partial_j y_m \partial_k y_m + \\
\quad \quad\quad\quad\frac{2y_m - 1}{y_m^2(1-y_m)^2} \partial_\mu y_m \partial_j y_m \partial_k y_m \bigg] \delta \theta_1^j \delta \theta_1^k\\
\frac{1}{2} \partial_m \partial_n \mcal{L}(t,f) \partial_j \partial_k z^m \partial_\mu z^n \delta\theta_1^j \delta\theta_1^k = \frac{1}{2}\bigg[ \frac{1}{y_m(1-y_m)} \partial_j \partial_k y_m \partial_\mu y_m \\
\quad \quad\quad\quad + \frac{2y_m -1}{y_m^2(1-y_m)^2}\partial_j y_m \partial_k y_m \partial_\mu y_m \bigg] \delta \theta_1^j \delta \theta_1^k\\
\partial_k \mcal{L}(t,f) \partial_\mu \partial_j z^k \delta \theta_1^j = \bigg[ \frac{1}{y_k(1-y_k)}\partial_\mu \partial_j y_k + \frac{2y_k -1}{y_k^2(1-y_k)^2} \partial_\mu y_k \partial_j y_k\bigg] ( y_k - t_k )\delta \theta_1^j
\end{gather*}
In this case, \eqref{eqn:smallcurve} will give a similar result as geodesic correction, which is $-\frac{1}{2}\Gamma_{jk}^{(1)\mu} \delta\theta_1^j \delta\theta_1^k$. The only difference is using 1-connection $\Gamma_{jk}^{(1)\mu}$~\citep{amari1987differential} instead of Levi-Civita connection $\Gamma_{jk}^\mu$.	
\end{proposition}

We also need an additional algorithm to calculate $\lambda_i \partial_\nu \partial_\alpha y_i \partial_\beta y_i \dot \theta^\alpha \dot \theta^\beta$, as provided by Alg.~\ref{alg:term3}.

\begin{algorithm}
	\caption{An algorithm for computing $\lambda_i \partial_\nu \partial_\alpha y_i \partial_\beta y_i \dot \theta^\alpha \dot \theta^\beta$ (term 3)}\label{alg:term3}
	\begin{algorithmic}[1]
		\Require{$\dot{\theta}$. We abbreviate $\mcal{R}_{\dot \theta}$ to $\mcal{R}$.}
		\item[]
		\State{$\mcal{R}(a_0) \gets 0$}
		\Comment{Since $a_0$ is not a function of the parameters}
		\item[]
		\For{$i\gets$ 1 to $l$}
		\Comment{forward pass}
		\State{$\mcal{R}(s_i) \gets \mcal{R}(W_i) a_{i-1} + W_i \mcal{R}(a_{i-1}) + \mcal{R}(b_i)$}
		\Comment{product rule}
		\State{$\mcal{R}(a_i) \gets \mcal{R}(s_i) \phi_i'(s_i)$}		
		\Comment{chain rule}
		\EndFor
		\item[]\Comment{By here we have computed all $\partial_\beta y_i \dot{\theta}^\beta$}
		\State{Compute $\vec{\lambda}$ from $a_l$}
		\State{
			Let $L = \sum_{i=1}^o \lambda_i \mcal{R}(a_l^i) y_i$
		}
		\State{$\mcal{R}(\mcal{D}(a_l)) \gets \mcal{R}\left( \frac{\partial L}{\partial \mbf{y}}\bigg|_{\mbf{y}=a_l} \right)=\frac{\partial^2 L}{\partial \mbf{y}^2}\bigg|_{\mbf{y}=a_l} \mcal{R}(a_l) = 0$}
		\State{$\mcal{D}a_l \gets \frac{\partial L}{\partial \mbf{y}}\bigg|_{\mbf{y}=a_l} = (\lambda_1(\mcal{R}a_l^1), \cdots, \lambda_o(\mcal{R}a_l^o))$}		
		\item[]
		\For{$i \gets$ $l$ to 1}
		\State{$\mcal{D}(s_i) \gets \mcal{D}(a_i) \odot \phi_i'(s_i)$}
		\State{$\mcal{D}(W_i) \gets \mcal{D}(s_i) a_{i-1}^\intercal$}
		\State{$\mcal{D}(b_i) \gets \mcal{D}(s_i)$}
		\State{$\mcal{D}(a_{i-1})\gets W_i^\intercal \mcal{D}(s_i)$}		
		\State{$
			\mcal{R}(\mcal{D}(s_i)) \gets \mcal{R}(\mcal{D}(a_i)) \odot \phi_i'(s_i) + \mcal{D}(a_i) \odot \mcal{R}(\phi_i'(s_i))
			=\mcal{R}(\mcal{D}(a_i))\odot \phi_i'(s_i) + \mcal{D}(a_i) \odot \phi_i''(s_i) \odot \mcal{R}(s_i)
			$}
		\State{$\mcal{R}(\mcal{D}(W_i)) \gets \mcal{R}(\mcal{D}(s_i)) a_{i-1}^\intercal + \mcal{D}(s_i) \mcal{R}(a_{i-1}^\intercal)$}
		\State{$\mcal{R}(\mcal{D}(b_i)) \gets \mcal{R}(\mcal{D}(s_i))$}
		\State{$\mcal{R}(\mcal{D}(a_{i-1})) \gets \mcal{R} (W_i^\intercal) \mcal{D}(s_i) + W_i^\intercal \mcal{R}(\mcal{D}(s_i))$}
		\EndFor
		\item[]
		\Return{$\lambda_i \partial_\nu \partial_\alpha y_i \partial_\beta y_i \dot \theta^\alpha \dot \theta^\beta = (\mcal{R}(\mcal{D}(W_i)),\cdots,\mcal{R}(\mcal{D}(W_l)), \mcal{R}(\mcal{D}(b_1)),\cdots,\mcal{R}(\mcal{D}(b_l)))$}
	\end{algorithmic}
\end{algorithm}

\subsection{Empirical Results}
If the small curvature approximation holds, \eqref{eqn:perturb} should perform similarly as \eqref{eqn:smallcurve}. The power of geodesic correction can thus be viewed as a higher-order approximation to the Hessian than natural gradient / Gauss-Newton and the interpretation of preserving higher-order invariance would be doubtful. 

In order to verify the small curvature approximation, we use both \eqref{eqn:perturb} (named ``\textbf{perturb}'') and \eqref{eqn:smallcurve} (named ``\textbf{geodesic}'') for correcting the natural gradient update. The difference between their performance shows how well small curvature approximation holds. Using the same settings in main text, we obtain the results on CURVES, MNIST and FACES. 

\begin{figure}
	\centering
	\includegraphics[width=0.6\textwidth]{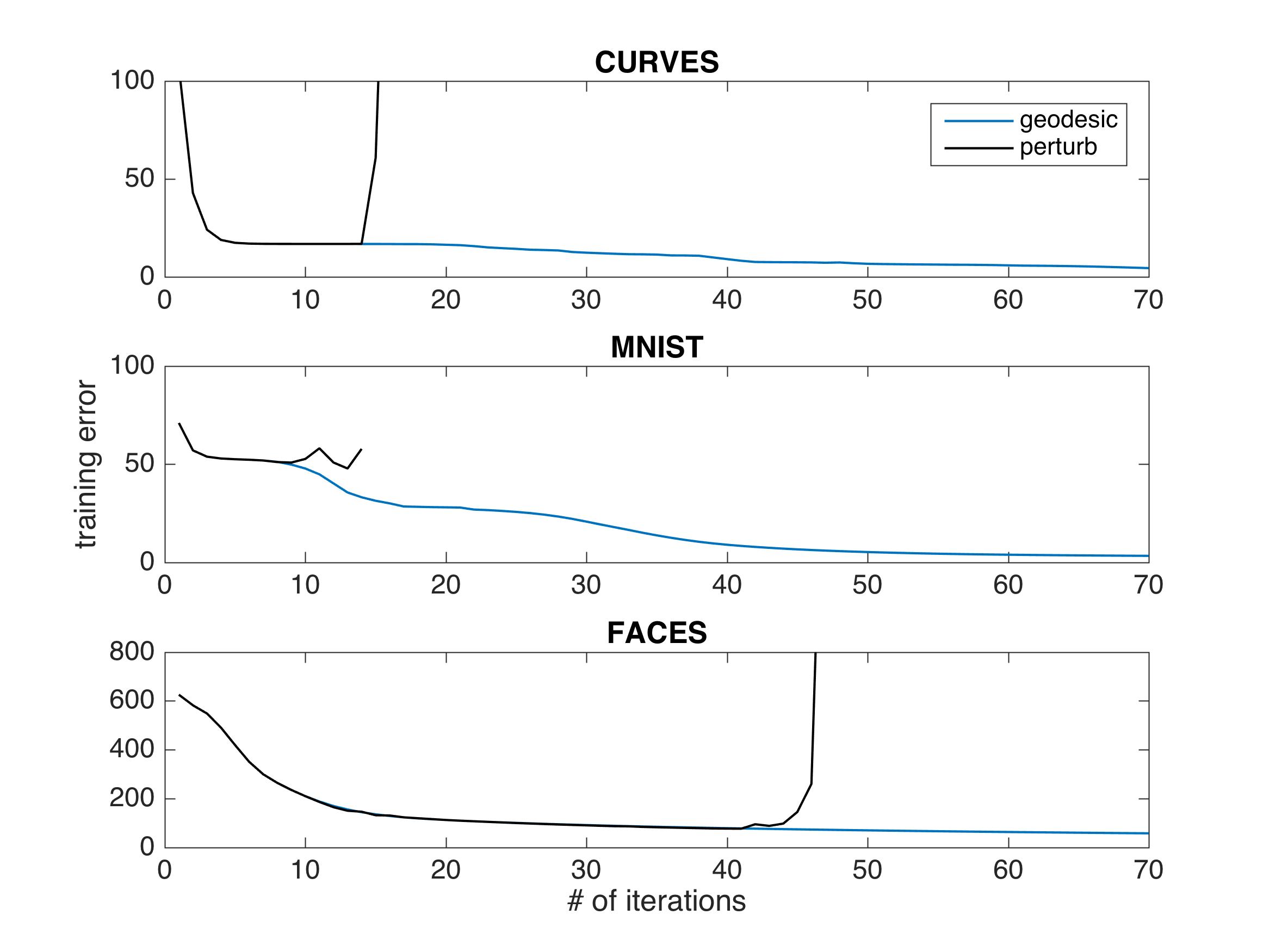}
		\caption{Study of small curvature approximation on different datasets. }\label{fig:perturb}
\end{figure}

As shown in \figref{fig:perturb}, \eqref{eqn:perturb} never works well except for the beginning. This is anticipated since Newton's method is susceptible to negative curvature and will blow up at saddle points~\citep{dauphin2014identifying}. Therefore the direction of approximating Newton's method more accurately is not reasonable. The close match of \eqref{eqn:perturb} and \eqref{eqn:smallcurve} indicates that the small curvature assumption indeed holds temporarily, and Newton's method is very close to natural gradient at the beginning. However, the latter divergence of \eqref{eqn:perturb} demonstrates that the effectiveness of geodesic correction does not come from approximating Newton's method to a higher order.

\end{document}